\documentclass[11pt]{article}

\usepackage[colorlinks=true,linkcolor=blue,urlcolor=black,citecolor=blue,breaklinks]{hyperref}
\usepackage{fullpage}
\usepackage[T1]{fontenc}
\usepackage{color}
\usepackage{enumerate}
\usepackage{graphicx}
\usepackage{algorithm,algcompatible}
\usepackage[font={small}]{caption}

\usepackage{amsmath,amsthm,amssymb,colonequals}

\newtheorem{thm}{Theorem}[section]

\newtheorem{lem}[thm]{Lemma}
\newtheorem{prop}[thm]{Proposition}

\newtheorem{cor}[thm]{Corollary}

\numberwithin{equation}{section}

\DeclareMathOperator*{\diag}{diag}
\DeclareMathOperator*{\Tr}{\mathbf{Tr}}

\newcommand{\R}{\ensuremath{\mathbb{R}}}

\newcommand{\F}{\ensuremath{\mathcal{F}}}

\newcommand{\norm}[1]{\lVert #1 \rVert}
\newcommand{\bignorm}[1]{\left\lVert #1 \right\rVert}

\newcommand{\ip}[2]{\ensuremath{\langle #1, #2 \rangle}}

\newcommand{\E}{\mathbb{E}}
\newcommand{\abs}[1]{\ensuremath{| #1 |}}

\newcommand{\floor}[1]{\left\lfloor #1 \right\rfloor}
\newcommand{\ceil}[1]{\left\lceil #1 \right\rceil}

\newcommand{\ind}{\mathbf{1}}
\renewcommand{\vec}{\mathrm{vec}}

\newcommand{\T}{\mathsf{T}}

\renewcommand{\Pr}{\mathbb{P}}

\newcommand{\svec}{\mathrm{svec}}
\newcommand{\smat}{\mathrm{smat}}

\newcommand{\vstacktwo}[2]{\begin{bmatrix} #1 \\ #2 \end{bmatrix}}
\newcommand{\bmattwo}[4]{\begin{bmatrix} #1 & #2 \\ #3 & #4 \end{bmatrix}}

\newcommand{\phitilde}{\widetilde{\phi}}

\newcommand{\Deltatilde}{\widetilde{\Delta}}
\newcommand{\hinf}{\mathcal{H}_\infty}
\newcommand{\hinfnorm}[1]{\norm{#1}_{\hinf}}

\newcommand{\calE}{\mathcal{E}}

\begin{document}

\title{Least-Squares Temporal Difference Learning for the Linear Quadratic Regulator}

\author{Stephen Tu and Benjamin Recht \\
University of California, Berkeley}
\maketitle

\begin{abstract}
Reinforcement learning (RL) has been successfully used to solve many continuous
control tasks. Despite its impressive results however, fundamental questions
regarding the sample complexity of RL on continuous problems remain open.
We study the performance of RL in this setting by
considering the behavior of the Least-Squares Temporal Difference (LSTD) estimator
on the classic Linear Quadratic Regulator (LQR) problem from optimal control.
We give the first finite-time analysis of the number of samples needed to
estimate the value function for a fixed static state-feedback policy to within
$\varepsilon$-relative error.
In the process of deriving our result, we give a general characterization
for when the minimum eigenvalue of the empirical covariance matrix formed
along the sample path of a fast-mixing stochastic process concentrates above zero, extending
a result by Koltchinskii and Mendelson~\cite{koltchinskii13} in the independent covariates setting.
Finally, we provide experimental evidence indicating that our analysis
correctly captures the qualitative behavior of LSTD on several LQR instances.
\end{abstract}

\section{Introduction}

Despite excellent performance on locomotion~\cite{kober13,levine14,lillicrap16,schulman16,tedrake04} and manipulation~\cite{krishnan17,levine16,levine16b,levine15} tasks,
model-free reinforcement learning (RL) is still considered very data intensive. This is especially a problem
for learning on robotic systems which requires human supervision, limiting the
applicability of RL.
While there have been various attempts to improve the sample efficiency of RL in practice~\cite{gu17,gu16,schaul16},
a theoretical understanding of the issue is still an open
question.  A more rigorous foundation could help to differentiate between whether RL
suffers from fundamental statistical limitations in the continuous setting, or
if more sample efficient estimators are possible.

For continuous control tasks, the Linear Quadratic Regulator (LQR) is an ideal
benchmark for studying RL, due to a combination of its theoretical tractability
combined with its practical application in various engineering domains.
Recent work by Dean et al.~\cite{dean17} adopts this point of view, and
studies the problem of designing a stabilizing controller for LQR
when the system dynamics are unknown to the practitioner.
Here, the authors take a model-based approach, and propose
to directly estimate the state-transition matrices that describe the dynamics from observations.
In practice however, model-free methods such as $Q$-learning or policy-gradient
type algorithms are preferred over model-based methods due to their flexibility and ease of use.
This naturally raises the question of how well
do model-free RL methods perform on the LQR problem.

In this paper, we shed light on this question by focusing on the classic Least-Squares Temporal Difference (LSTD)
estimator~\cite{boyan99,bradtke96}.
Given a sample trajectory from a Markov Decision Process (MDP) in
feedback with a fixed policy $\pi$,
LSTD
computes the value function $V^\pi$ associated to $\pi$.
Estimating $V^\pi$ is the core primitive in value and policy-iteration type algorithms~\cite{sutton98}.
The key property exploited by LSTD is the \emph{linear-architecture} assumption, which states that
the value function can be expressed as a
linear function after applying a known non-linear transformation to the state.
To the best of our knowledge, LQR is the simplest continuous problem
which exhibits this property.

Our main result regarding the LSTD estimator for LQR is an upper bound on the
necessary length of a single trajectory to estimate the value function of a
stabilizing state-feedback policy.  Letting $n$
denote the dimension of the state and ignoring instance specific factors, we
establish that roughly $n^2/\varepsilon^2$ samples are sufficient to estimate
the value function up to $\varepsilon$-relative error.
Our analysis builds upon the work of Lazaric et al.~\cite{lazaric12},
which requires bounding the minimum eigenvalue of the sample covariance matrix
formed by the transformed state vectors; the same eigenvalue quantity also appears in many other analyses of the LSTD
estimator in the literature~\cite{lazaric12,liu15,liu12,prashanth14}.
We bound this quantity by studying the more general problem of controlling the
minimum eigenvalue of the covariance matrix formed from dependent covariates
that mix quickly to a stationary distribution. Our analysis extends an elegant technique based on small-ball
probabilities from Koltchinskii and Mendelson~\cite{koltchinskii13}, and is of independent interest.
Specializing to the setting when the covariates are bounded almost surely,
our result improves upon the analysis given by Lazaric et al.

We conclude our work with an end-to-end empirical comparison of the model-free
Least-Squares Policy Iteration (LSPI) algorithm~\cite{lagoudakis03} with
the model-based methods proposed in Dean et al.  Our experiments show that
model-free LSPI can be substantially less
sample efficient and less robust compared to model-based methods.
This corroborates our theoretical results which suggest a factor of state-dimension gap
between the number of samples needed to estimate a value function versus
the bounds given in Dean et al.\ for robustly computing a stabilizing controller.
We hope that our findings encourage further investigation, both theoretical and
empirical, into the performance of RL on continuous control problems.

\subsection{Related Work}

Least-squares methods for temporal
difference learning are well-studied in reinforcement learning,  
with asymptotic convergence results in a
general MDP setting provided by~\cite{tsitsiklis97,yu09}.  More recently,
non-asymptotic analyses were given in both the batch
setting~\cite{antos08,farahmand16,lazaric12} and the online
setting~\cite{liu15,liu12,prashanth14}.  The prevailing assumption employed in prior
art is that the MDP has uniformly bounded features and rewards, which excludes
the LQR problem. 
We note that earlier results by Bradtke~\cite{bradtke93,bradtke94} studied
policy-iteration specifically for LQR, and proved an asymptotic 
convergence result.
To the best of our knowledge, our work is the first to 
provide finite-time results for temporal difference learning on LQR.
Furthermore, our concentration result for the sample covariance matrix
drawn from a mixing process specialized to the bounded setting 
improves upon Lemma 4 of Lazaric et al.~\cite{lazaric12}, by reducing the necessary trajectory
length from $\Omega(d^2)$ to $\Omega(d)$, where $d$ is the dimension of the
features.

The problem of estimating the spectra of an empirical covariance matrix formed
from independent samples has received much attention in the past decade.  Some
representative results can be found in
\cite{adamczak11,koltchinskii13,mendelson14,rudelson09,srivastava13,vershynin11}
and the references within. Our focus on the result of Koltchinskii and Mendelson in
this paper is primarily motivated by the fact that their proof technique is
generalizable to the dependent-data setting using standing mixing assumptions.
The use of distributional mixing assumptions for proving uniform convergence bounds
is by now a well-established technique in the statistics and machine learning literature;  see
\cite{mohri08,mohri10,yu94} for some of the earlier results, and
\cite{agarwal13,kuznestov15,kuznestov16,mcdonald17b} for generalizations to
time-series and online learning. In this work, our focus is on bounding a very
particular empirical process (the minimum eigenvalue of a sample covariance matrix), 
and not in developing general machinery for
empirical process theory on dependent data.

\section{A Sample Covariance Bound for Fast-Mixing Processes}
\label{sec:bound}

In this section, we state and prove our result regarding the minimum eigenvalue
of the sample covariance matrix formed along a trajectory of a $\beta$-mixing process.
We start by fixing notation.
Let $(X_k)_{k=1}^{\infty}$ be an $\R^n$-valued discrete-time stochastic process
adapted to a filtration $(\F_k)_{k=1}^{\infty}$.
For all $k \geq 1$, let $\nu_k$ denote the marginal
distribution of $X_k$. We assume that $(X_k)_{k=1}^{\infty}$ admits a stationary distribution
$\nu_\infty$, and we define the $\beta$-mixing coefficient $\beta(k)$ with respect to $\nu_\infty$ as
\begin{align}
  \beta(k) := \sup_{t \geq 1} \E_{X_1^t}[ \norm{ \Pr_{X_{t+k}}(\cdot | \F_t) - \nu_\infty }_{\mathrm{tv}} ] \label{eq:beta_mixing_def} \:.
\end{align}
In \eqref{eq:beta_mixing_def}, the notation $X_1^t$ refers to the prefix $X_1^t := (X_1, ..., X_t)$
and $\norm{\cdot}_{\mathrm{tv}}$ refers to the total-variation
norm on probability measures.
Our main assumption in what follows is that
$(X_k)_{k=1}^{\infty}$ is $\beta$-mixing
to its stationary distribution at an exponential decay
rate, i.e. $\beta(k) \leq \Gamma \rho^k$ for some fixed $\Gamma > 0$ and $\rho \in (0, 1)$.
We note that our analysis is easily amendable to slower (e.g. polynomial) decay rates.

We consider a sample-path $X_1, X_2, ...$ drawn from this stochastic process.
Fix positive integers $N$ and $a$ satisfying $1 \leq a \leq N$ and define:
\begin{align}
  X_{(j)}^{N} := ( X_k : 1 \leq k \leq N, (k - 1\ \mathrm{mod}\ a) = j-1 ) \:, \:\: j=1, ..., a \:.
\end{align}
Let the integers $m_1, ..., m_a$
(resp. index sets $I_{(1)}, ..., I_{(a)}$)
denote the sizes (resp. indices)
of $X_{(1)}^{N}, ..., X_{(a)}^{N}$.
Also, let $X_\infty^{m_j}$ be $m_j$ i.i.d. draws from the stationary distribution
$\nu_\infty$. With this notation in hand, the following lemma is one of the standard ways to utilize
mixing assumptions in analysis.
\begin{lem}[Proposition 2, Kuznetsov and Mohri~\cite{kuznetsov17}]
\label{lemma:subsampling}
Let $g$ be a real-valued Borel measurable function satisfying $0 \leq g \leq 1$. Then,
for all $j=1, ..., a$,
\begin{align*}
  \abs{\E[g(X_\infty^{m_j})] - \E[g(X_{(j)}^N)]} \leq m_j \beta(a) \:,
\end{align*}
where $\beta(a)$ is defined in \eqref{eq:beta_mixing_def}.
\end{lem}
In our analysis, we will take the $g$ in Lemma~\ref{lemma:subsampling} to be
the indicator function on an event, which will allow us to relate
events on the stochastic process $X_\infty^{m_j}$ to events on the blocked version
of the stochastic process $X_{(j)}^N$.
We are now ready to prove our generalization of Theorem 2.1 from
Koltchinskii and Mendelson~\cite{koltchinskii13} for
fast-mixing processes.
We note that no attempt was made to optimize the constants appearing in the result.
\begin{thm}
\label{thm:mendelson_markov}
Fix a $\delta \in (0, 1)$. Suppose that $(X_k)_{k=1}^{\infty}$ is a discrete-time
stochastic process with stationary distribution $\nu_\infty$
that satisfies $\beta(a) \leq \Gamma \rho^a$ for some $\Gamma > 0$, $\rho \in (0, 1)$,
where $\beta(a)$ is defined in \eqref{eq:beta_mixing_def}.
For any positive $\tau > 0$ define
the small-ball probability $Q_\infty(\tau)$ as
\begin{align}
  Q_\infty(\tau) := \inf_{t \in S^{n-1}} \Pr_{\nu_\infty} \{ \abs{\ip{t}{X}} \geq \tau \} \:. \label{eq:small_ball}
\end{align}
Suppose that there exists a $\tau$ satisfying $Q_\infty(\tau) > 0$.
Furthermore, suppose that $N$ satisfies
\begin{align}
  N &\geq  \frac{1}{1-\rho} \log\left(\frac{2 \Gamma N}{\delta}\right)
  \left(\max\left\{  \frac{1024 \E_{\nu_\infty}[ \norm{X}^2 ]}{\tau^2 Q^2_\infty(\tau) } ,
  \frac{32}{Q_\infty^2(\tau)} \log\left( \frac{4}{\delta(1-\rho)} \log\left(\frac{2 \Gamma N}{\delta}\right) \right)
  \right\} + 1\right) \label{eq:N_assumption} \:.
\end{align}
Then, with probability at least $1-\delta$,
\begin{align*}
    \lambda_{\min}\left( \frac{1}{N} \sum_{k=1}^{N} X_kX_k^\T \right) \geq \frac{\tau^2 Q_\infty(\tau) }{8} \:.
\end{align*}
\end{thm}
\begin{proof}
The first part of this proof follows the argument
of Theorem 2.1 from \cite{koltchinskii13}. Hence, we adopt their notation.
Fix an arbitrary function class $\F$ of functions mapping $\R^n$ to $\R$.
We associate to $\F$ the function $Q_\infty(\tau; \F)$ defined for a positive parameter $\tau$ as
\begin{align}
    Q_\infty(\tau; \F) := \inf_{f \in \F} \Pr_{\nu_\infty}\{ \abs{f(X)} \geq \tau \} \:.
\end{align}
Next, define $\phi_u : \R_+ \longrightarrow [0, 1]$ as
\begin{align*}
	\phi_u(t) := \begin{cases} 1 &\text{if } t \geq 2u \\
		(t/u) - 1 &\text{if } u \leq t \leq 2u \\
		0 &\text{if } t < u
	\end{cases} \:.
\end{align*}
It has the property that for all $t \in \R$, $\ind_{[u, \infty)}(t) \geq \phi_u(t)$ and
$\phi_u(t) \geq \ind_{[2u, \infty)}(t)$.

Now let $\F$ be an arbitrary function class, and fix an $f \in \F$.
Clearly, we have $\abs{f(X_k)}^2 \geq u^2 \ind_{\abs{f(X_k)} \geq u}$ for $k=1, ..., N$.
Therefore,
\begin{align}
	\frac{1}{N} \sum_{k=1}^{N} \abs{f(X_k)}^2
    &\geq \frac{1}{N} \sum_{k=1}^{N} u^2 \ind_{\abs{f(X_k)} \geq u} \nonumber \\
	&= \frac{1}{N} \sum_{k=1}^{N} u^2 (\nu_\infty\{ \abs{f } \geq 2 u\} + \ind_{\abs{f(X_k)} \geq u} - \nu_\infty\{ \abs{f } \geq 2 u\}) \nonumber \\
	&\geq \frac{1}{N} \sum_{k=1}^{N} u^2( Q_\infty(2u; \F ) + \ind_{\abs{f(X_k)} \geq u} - \nu_\infty\{ \abs{f } \geq 2 u\}  ) \nonumber \\
	&\geq \frac{1}{N} \sum_{k=1}^{N} u^2( Q_\infty(2u; \F ) + \phi_{u}(\abs{f(X_k)}) - \E_{\nu_\infty}[ \phi_{u}(\abs{f }) ]  ) \nonumber \\
	&= u^2 \left( Q_\infty(2u; \F ) + \frac{1}{N} \sum_{k=1}^{N} ( \phi_{u}(\abs{f(X_k)}) - \E_{\nu_\infty}[ \phi_{u}(\abs{f }) ] ) \right) \nonumber \\
	&\geq u^2 \left( Q_\infty(2u; \F ) + \inf_{f \in \mathcal{F}} \frac{1}{N} \sum_{k=1}^{N} (\phi_{u}(\abs{f(X_k)}) - \E_{\nu_\infty}[ \phi_{u}(\abs{f }) ]   ) \right) \nonumber \\
    &= u^2 \left( Q_\infty(2u; \F ) - \sup_{f \in \mathcal{F}} \frac{1}{N} \sum_{k=1}^{N} ( \E_{\nu_\infty}[ \phi_{u}(\abs{f }) ]   - \phi_{u}(\abs{f(X_k)})  ) \right) \:. \label{eq:mendelson_bound}
\end{align}
Since $f \in \F$ is arbitrary, \eqref{eq:mendelson_bound} holds for
$\inf_{f \in \F} \frac{1}{N} \sum_{k=1}^{N} \abs{f(X_k)}^2$.
The rest of the proof is devoted to upper bounding the empirical process
\begin{align*}
    \sup_{f \in \mathcal{F}} \frac{1}{N} \sum_{k=1}^{N} ( \E_{\nu_\infty}[ \phi_{u}(\abs{f }) ]   - \phi_{u}(\abs{f(X_k)})  ) \:.
\end{align*}
To do this, we will partition our $X_k$'s into $a$ groups
$X_{(j)}^{N}$, $j=1, ..., a$, where we define $a$ to be
\begin{align}
  a := \ceil{\frac{1}{1-\rho} \log\left( \frac{2 \Gamma N}{\delta} \right) } \:. \label{eq:def_a}
\end{align}
With this in mind, we write,
\begin{align}
  &\sup_{f \in \mathcal{F}} \frac{1}{N} \sum_{k=1}^{N} ( \E_{\nu_\infty}[ \phi_{u}(\abs{f }) ]   - \phi_{u}(\abs{f(X_k)})  ) \nonumber \\
  &\qquad= \sup_{f \in \mathcal{F}} \frac{1}{a} \sum_{j=1}^{a} \frac{m_j (a/N)}{ m_j } \sum_{k \in I_{(j)}} ( \E_{\nu_\infty}[ \phi_{u}(\abs{f }) ]   - \phi_{u}(\abs{f(X_k)})  ) \nonumber \\
  &\qquad\leq \frac{1}{a} \sum_{j=1}^{a} m_j(a/N) \sup_{f \in \F} \frac{1}{m_j} \sum_{k \in I_{(j)}} ( \E_{\nu_\infty}[ \phi_{u}(\abs{f }) ]   - \phi_{u}(\abs{f(X_k)})  ) \label{eq:sup_convexity} \:.
\end{align}
By the definition of the $m_j$'s, we know that
\begin{align*}
  \underline{m} := \floor{ \frac{N}{a} } \leq m_j \leq \frac{N}{a} + 1 \:,
\end{align*}
and therefore each $m_j(a/N) \leq 2$.
Setting $\varepsilon$ to
\begin{align}
  \varepsilon = \max_{j=1, ..., a} \E_{X_\infty^{m_j}}\left[ \sup_{f \in \F} \frac{1}{m_j} \sum_{k=1}^{M} (\E_{\nu_\infty}[ \phi_u(\abs{f }) ] - \phi_u(\abs{f(X_k)}) ) \right] + \sqrt{\frac{1}{2\underline{m}} \log(2a/\delta)} \:, \label{eq:epsilon_setting}
\end{align}
we have by combining \eqref{eq:sup_convexity} with a union bound,
\begin{align}
    &\Pr_{X_1^N}\left\{\sup_{f \in \mathcal{F}} \frac{1}{N} \sum_{k=1}^{N} ( \E_{\nu_\infty}[ \phi_{u}(\abs{f }) ] - \phi_{u}(\abs{f(X_k)})  )  > 2\varepsilon \right\} \nonumber \\
    &\qquad\leq \sum_{j=1}^{a} \Pr_{X_{(j)}^N}\left\{ \sup_{f \in \F} \frac{1}{m_j} \sum_{k \in I_{(j)}} ( \E_{\nu_\infty}[ \phi_{u}(\abs{f }) ]   - \phi_{u}(\abs{f(X_k)}) > \varepsilon \right\} \nonumber \\
    &\qquad\stackrel{(a)}{\leq} N\beta(a) + \sum_{j=1}^{a} \Pr_{X_\infty^{m_j}}\left\{ \sup_{f \in \F} \frac{1}{m_j} \sum_{k=1}^{m_j} (\E_{\nu_\infty}[\phi_u(\abs{f })] - \phi_u(\abs{f(X_k)}) > \varepsilon \right\} \nonumber \\
    &\qquad\stackrel{(b)}{\leq} N\beta(a) + \delta/2 \nonumber \\
    &\qquad\stackrel{(c)}{\leq} N \Gamma \rho^a + \delta/2 \nonumber \\
    &\qquad\stackrel{(d)}{\leq} \delta/2 + \delta/2 = \delta \:. \label{eq:probability_bound}
\end{align}
The inequality (a) follows from Lemma~\ref{lemma:subsampling},
(b) holds by the bounded differences inequality, since $X_\infty^{m_j}$ contains $m_j$ i.i.d. datapoints,
(c) uses the assumption on $\beta(a)$, and (d) follows by the definition of $a$ in \eqref{eq:def_a}.
Furthermore, using the fact that $\phi_u(\abs{\cdot})$ is $1/u$-Lipschitz, we bound the expected supremum
via the standard symmetrization inequality,
\begin{align}
  &\max_{j=1, ..., a} \E_{X_\infty^{m_j}}\left[ \sup_{f \in \F} \frac{1}{m_j} \sum_{k=1}^{m_j} (\E_{\nu_\infty}[ \phi_u(\abs{f }) ] - \phi_u(\abs{f(X_k)}) ) \right] \nonumber \\
  &\qquad\leq \max_{j=1, ...,a} 2 \mathcal{R}_{m_j}(\phi_u(\abs{\cdot}) \circ \F ) \leq \max_{j=1, ...,a} \frac{2}{u} \mathcal{R}_{m_j}(\F ) \:. \label{eq:expected_supremum_bound}
\end{align}
Above, $\mathcal{R}_{m_j}$ denotes the Rademacher complexity,
\begin{align*}
  \mathcal{R}_{m_j}(\F) := \E_{\varepsilon_1^{m_j}, X_\infty^{m_j}}\left[ \sup_{f \in \F} \frac{1}{m_j} \sum_{k=1}^{m_j} \varepsilon_k f(X_k) \right] \:.
\end{align*}
In view of \eqref{eq:epsilon_setting}, \eqref{eq:probability_bound}, and \eqref{eq:expected_supremum_bound},
with probability at least $1-\delta$,
\begin{align}
    \sup_{f \in \mathcal{F}} \frac{1}{N} \sum_{k=1}^{N} ( \E_{\nu_\infty}[ \phi_{u}(\abs{f }) ]   - \phi_{u}(\abs{f(X_k)})  ) \leq
        \max_{j=1, ..., a} \frac{4}{u} \mathcal{R}_{m_j}(\F ) + 2\sqrt{\frac{1}{2m} \log(2a/\delta)} \:.
\end{align}
Combining this inequality with \eqref{eq:mendelson_bound}, if
\begin{align}
  \max_{j=1, ..., a} \mathcal{R}_{m_j}(\F ) &\leq \frac{u}{16} Q_\infty(2u; \F ) \:, \label{eq:rademacher_bound} \\
    \underline{m} &\geq \frac{32\log(2a/\delta)}{Q_\infty^2(2u; \F )} \:, \label{eq:m_cond1}
\end{align}
then on this event we have
\begin{align*}
    \inf_{f \in \F} \frac{1}{N} \sum_{k=1}^{N} \abs{f(X_k)}^2 \geq \frac{u^2}{2} Q_\infty(2u;\F ) \:.
\end{align*}
Now we specialize to $\F = \{ \ip{\cdot}{t} : t \in S^{n-1} \}$,
for which
\begin{align*}
    \inf_{f \in \F} \frac{1}{N} \sum_{k=1}^{N} \abs{f(X_k)}^2 = \lambda_{\min}\left( \frac{1}{N} \sum_{k=1}^{N} X_kX_k^\T \right) \:.
\end{align*}
In this case,
\begin{align*}
  \mathcal{R}_{m_j}(\F ) &= \E_{\varepsilon_1^{m_j}, X_\infty^{m_j}}\bignorm{ \frac{1}{m_j} \sum_{k=1}^{m_j} \varepsilon_k X_k } \leq \frac{1}{m_j} \sqrt{\E_{\varepsilon_1^{m_j}, X_\infty^{m_j}} \bignorm{ \sum_{k=1}^{m_j} \varepsilon_k X_k }^2 } \\
    &= \frac{1}{m_j} \sqrt{ m_j \E_{\nu_\infty}[\norm{ X }^2] } = \frac{\sqrt{\E_{\nu_\infty}[ \norm{X}^2 ]}}{\sqrt{m_j}} \:.
\end{align*}
Hence, \eqref{eq:rademacher_bound} is satisfied if
\begin{align}
  \underline{m} \geq \frac{256 \E_{\nu_\infty}[ \norm{X}^2 ]}{u^2 Q^2_\infty(2 u ; \F ) } \:. \label{eq:m_cond2}
\end{align}
We now verify that these two inequalities on $\underline{m}$ are indeed valid.
Using the fact that for any real number $x$ we have $\floor{x} \geq x - 1$ and $x \leq \ceil{x}$,
\begin{align*}
  \underline{m} &= \floor{\frac{N}{a}} \geq \frac{N}{a} - 1 \geq N \frac{1-\rho}{\log(2 \Gamma N/\delta)} - 1 \\
  &\geq \max\left\{  \frac{1024 \E_{\nu_\infty}[ \norm{X}^2 ]}{\tau^2 Q^2_\infty(\tau) } , \frac{32\log(2a/\delta)}{Q_\infty^2(\tau)}  \right\} \:,
\end{align*}
where the last inequality holds from the assumption on $N$ in \eqref{eq:N_assumption}.
By performing a change of variables $\tau \gets 2u$, we see that
\eqref{eq:m_cond1} and \eqref{eq:m_cond2} both hold.
\end{proof}

Following a similar line of reasoning as in Koltchinskii and Mendelson, we
immediately recover a corollary to Theorem~\ref{thm:mendelson_markov},
where the small-ball condition in \eqref{eq:small_ball} is replaced by
a stronger moment contractivity assumption.
\begin{cor}
\label{cor:mendelson_markov}
Fix a $\delta \in (0, 1)$. Suppose that $(X_k)_{k=1}^{\infty}$ is a discrete-time
stochastic process as described in the hypothesis of Theorem~\ref{thm:mendelson_markov}.
For $X$ drawn from the stationary measure $\nu_\infty$, suppose that
the following conditions hold,
\begin{align}
  0 < \ell \leq \lambda_{\min}(\E[XX^\T]) \leq L \:, \:\: \sup_{t \in S^{n-1}} \frac{\norm{\ip{X}{t}}_{L^2}}{\norm{\ip{X}{t}}_{L^1}} \leq B \:. \label{eq:mendelson_assumptions_explicit}
\end{align}
Furthermore, suppose that $N$ satisfies
\begin{align*}
  N &\geq  \frac{1}{1-\rho} \log\left(\frac{2 \Gamma N}{\delta}\right)
  \left(\max\left\{ 65536 B^6 \frac{L}{\ell} n  ,
  512 B^4 \log\left( \frac{4}{\delta(1-\rho)} \log\left(\frac{2 \Gamma N}{\delta}\right) \right)
  \right\} + 1\right) \:.
\end{align*}
Then, with probability at least $1-\delta$,
\begin{align*}
  \lambda_{\min}\left( \frac{1}{N} \sum_{k=1}^{N} X_kX_k^\T \right) \geq \frac{\ell}{128 B^4} \:.
\end{align*}
\end{cor}
\begin{proof}
The assumptions of \eqref{eq:mendelson_assumptions_explicit} imply that
$\E_{\nu_\infty}[\norm{X}^2] \leq L n$ and that for $\tau =
\frac{\sqrt{\ell}}{2 B}$, we have $Q_\infty(\tau) \geq \frac{1}{4B^2}$.
The claim now follows by Theorem~\ref{thm:mendelson_markov}.
\end{proof}

\section{Fast-Mixing of Linear Dynamical Systems}
\label{sec:mixing_lds}

In order to pave the way for our main result regarding LQR, we need to
understand the mixing time of a stable linear, time-invariant (LTI) dynamical system.
This will allow us to directly apply the results from Section~\ref{sec:bound}.
Consider the LTI system
\begin{align}
   X_{k+1} = A X_k + W_k \:, \:\: W_k \sim \mathcal{N}(0, I) \:, \label{eq:lti_system}
\end{align}
with $A$ an $n \times n$ matrix, initial condition $X_0 = 0$, and $W_k$ independent from $W_{k'}$ for all $k \neq k'$.
This section is dedicated towards bounding the $\beta$-mixing coefficient
of \eqref{eq:lti_system}.

It is not hard to see that the marginal distribution $\nu_k$ of $X_k$ evolving
according to \eqref{eq:lti_system}
is $\mathcal{N}(0, P_k)$, where the covariance $P_k := \sum_{t=0}^{k-1} (A^t)(A^t)^\T$ is positive-definite.
The stability of the linear system \eqref{eq:lti_system}
is equivalent to the spectral radius of $A$, denoted $\rho(A)$, being strictly less than one.
When $\rho(A) < 1$,
the stationary distribution $\nu_\infty$ of $(X_k)_{k=1}^{\infty}$ is $\mathcal{N}(0, P_\infty)$,
where the covariance matrix $P_\infty$ is
the unique, positive-definite solution of the
discrete-time Lyapunov equation
\begin{align}
  A P_\infty A^\T - P_\infty + I = 0 \:.
\end{align}
Observe that in the case of a Markov chain, the $\beta$-mixing coefficient \eqref{eq:beta_mixing_def} simplifies to
\begin{align*}
    \beta(k) = \sup_{t \geq 1} \E_{x\sim \nu_t}[ \norm{\Pr_{X_k}(\cdot|X_0{=}x) - \nu_\infty}_{\mathrm{tv}} ] \:. \label{eq:beta_coeff_markov}
\end{align*}
The following upper bound on
$\E_{x\sim \nu_t}[ \norm{\Pr_{X_k}(\cdot|X_0{=}x) - \nu_\infty}_{\mathrm{tv}} ]$
uses the assumption of a known decay on the spectral
norm of $A^k$.
\begin{prop}
\label{prop:mixing_coeff}
Suppose that $\norm{A^k} \leq \Gamma \rho^k$ for all $k \geq 0$, where
$\Gamma > 0$ and $\rho \in (0, 1)$.
Let $\Pr_{X_k}(\cdot|X_0{=}x)$ denote the conditional distribution of $X_k$ given $X_0 = x$.
We have that for all $k \geq 0$ and any distribution $\nu_0$ over $x$,
\begin{align*}
	\E_{x \sim \nu_0}[\norm{ \Pr_{X_k}(\cdot| X_0{=}x) - \nu_\infty }_{\mathrm{tv}}] \leq \frac{\Gamma}{2} \sqrt{\E_{\nu_0}[\norm{x}^2] + \frac{n}{1-\rho^2}} \rho^{k} \:.
\end{align*}
\end{prop}
\begin{proof}
By Pinsker's inequality,
\begin{align*}
	\norm{\Pr_{X_k}(\cdot | X_0{=}x) - \nu_\infty}_{\mathrm{tv}} \leq \sqrt{\frac{1}{2} D(\Pr_{X_k}(\cdot | X_0{=}x), \nu_\infty)} \:,
\end{align*}
where $D(\cdot, \cdot)$ denotes the KL-divergence.
It is easy to check that $\Pr_{X_k}(\cdot|X_0{=}x) = \mathcal{N}(A^k x, P_k)$.
Using the formula for KL-divergence between two multivariate Gaussians,
\begin{align*}
	D(\Pr_{X_k}(\cdot | X_0{=}x), \nu_\infty) = \frac{1}{2} \left( \Tr(P_\infty^{-1} P_k) + x^\T (A^k)^\T P_\infty^{-1} A^k x - n + \log(\det(P_\infty)/\det(P_k))   \right) \:.
\end{align*}
Now, write $\Delta_k := P_\infty - P_k$, where $\Delta_k \succcurlyeq 0$.
We have,
\begin{align*}
	\frac{\det(P_\infty)}{\det(P_k)} = \det(P_\infty P_k^{-1}) = \det( (P_k + \Delta_k) P_k^{-1} ) = \det(I + \Delta_k P_k^{-1}) = \det(I + \Delta_k^{1/2} P_k^{-1} \Delta_k^{1/2})\:.
\end{align*}
	Therefore, using the inequality that $\log\det(A) \leq \Tr(A - I)$ for any positive definite $A$,
\begin{align*}
	\log\det(I + \Delta_k^{1/2} P_k^{-1} \Delta_k^{1/2}) \leq \Tr( \Delta_k^{1/2} P_k^{-1} \Delta_k^{1/2} ) = \Tr( \Delta_k P_k^{-1}) \:.
\end{align*}
On the other hand,
\begin{align*}
	\Tr(P_\infty^{-1}P_k) = \Tr(P_\infty^{-1}(P_\infty - \Delta_k)) = \Tr(I - P_\infty^{-1}\Delta_k) = n - \Tr(P_\infty^{-1} \Delta_k) \leq n \:.
\end{align*}
Hence, combining these inequalities and letting
$\norm{x}_M := \sqrt{x^\T M x}$ for a positive-definite matrix $M$ and
$\norm{M}_*$ denote the nuclear norm of a matrix $M$,
\begin{align*}
	D(\Pr_{X_k}(\cdot | X_0{=}x), \nu_\infty)
	\leq \frac{1}{2}\left( x^\T (A^k)^\T P_\infty^{-1} A^k x + \Tr(\Delta_k P_k^{-1}) \right)
	\leq \frac{1}{2}(\norm{A^k x}^2_{P_\infty^{-1}} + \norm{P_k^{-1}}_* \norm{\Delta_k}) \:,
\end{align*}
where the last inequality follows
from von Neumann's trace inequality combined with H{\"o}lder's inequality.
Using the decay assumption,
\begin{align*}
	\norm{\Delta_k} = \bignorm{\sum_{t=k}^{\infty} A^t (A^k)^\T} \leq \Gamma^2 \sum_{t=k}^{\infty} \rho^{2t} = \frac{\Gamma^2 \rho^{2k}}{1-\rho^2} \:.
\end{align*}
Furthermore, since $P_\infty \succcurlyeq P_k \succcurlyeq I$, we have that
$\norm{P_k^{-1}}_* \leq n$ and $\norm{A^k x}^2_{P_\infty^{-1}} \leq \norm{A^k x}^2 \leq \Gamma^2 \norm{x}^2 \rho^{2k}$.
This gives the bound
\begin{align*}
	D(\Pr_{X_k}(\cdot | X_0{=}x), \nu_\infty) \leq \frac{1}{2} \left(\Gamma^2 \norm{x}^2 \rho^{2k} + \frac{n\Gamma^2}{1-\rho^2} \rho^{2k}\right) \:.
\end{align*}
The claim now follows by Jensen's inequality.
\end{proof}

Now we turn our attention to obtaining
a quantitative handle on the decay rate of the spectral norm of $A^k$.
To do this, we introduce some basic concepts from robust control theory;
see \cite{zhou95} for a more thorough treatment.
Let $\mathbb{T}$ (resp. $\mathbb{D}$) denote the unit circle (resp. open unit disk)
in the complex plane.
Let $\mathcal{RH}_\infty$ denote the space of matrix-valued, real-rational
functions which are analytic on $\mathbb{D}^c$.
For a $G \in \mathcal{RH}_\infty$,
we define the $\hinf$-norm $\hinfnorm{G}$ as
\begin{align}
  \hinfnorm{G} := \sup_{z \in \mathbb{T}} \: \norm{G(z)} \:.
\end{align}
Furthermore, given a square matrix $A$, we define its resolvant $\Phi_A(z)$ as
\begin{align}
  \Phi_A(z) := (z I - A)^{-1} \:.
\end{align}
When $A$ is stable, $\Phi_A \in \mathcal{RH}_\infty$, and hence $\hinfnorm{G} < \infty$.
The next proposition characterizes the decay rate in terms of the stability radius $\rho(A)$
and the $\hinf$-norm $\hinfnorm{\Phi_A}$. While the result is standard, we include its proof
for completeness.
\begin{prop}[See e.g. Lemma 1 from~\cite{goldenshluger01}]
\label{prop:spectral_decay}
Let $A$ be a stable matrix with spectral radius $\rho(A)$.
Fix any $\rho \in (\rho(A), 1)$. For all $k \geq 1$, we have
\begin{align}
  \norm{A^k} \leq \hinfnorm{\Phi_{\rho^{-1} A}} \rho^k \:.
\end{align}
\end{prop}
\begin{proof}
We first prove the following claim. Let $G \in \mathcal{RH}_\infty$
with stability radius $\rho(G) \in (0, 1)$.
Fix any $\rho \in (\rho(G), 1)$, and
write $G(z)$ in its power-series expansion $G(z) = \sum_{k=0}^{\infty} A_k z^{-k}$.
Then, for all $k \geq 1$, we have
\begin{align}
  \norm{A_k} \leq \hinfnorm{G(\rho z)} \rho^k \:. \label{eq:spectral_decay}
\end{align}
Fix two vectors $u, v \in S^{n-1}$.
Define the function $H_{uv}(z) := u^* G(z^{-1}) v$, which is analytic for all $\abs{z} \leq 1/\rho$.
It is easy to check that $k$-th derivative $H_{uv}^{(k)}(0) = k! u^* A_k v$.
Therefore,
\begin{align*}
  k!\abs{u^* A_k v} &= \abs{H_{uv}^{(k)}(0)} \leq k! \rho^{k} \max_{\abs{z} \leq 1/\rho} \abs{H_{uv}(z)} = k! \rho^k \max_{\abs{z} \geq \rho} \abs{u^* G(z) v} \\
  &= k!\rho^k \max_{\abs{z} \geq 1} \abs{u^* G(\rho z) v}
  \leq k!\rho^k \max_{\abs{z} \geq 1} \norm{ G(\rho z) }
  = k! \rho^k \hinfnorm{G(\rho z)} \:.
\end{align*}
Above, the first inequality is Cauchy's estimate formula for analytic functions,
and the last equality follows from the maximum modulus principle.
Since the upper bound is independent of $u,v$, we can take the supremum
over all $u,v \in S^{n-1}$ and reach the conclusion \eqref{eq:spectral_decay}.

We now apply this claim to the resolvant $\Phi_{A}$,
which has the series expansion $\Phi_{A}(z) = \sum_{k=1}^{\infty} A^{k-1} z^{-k}$.
For any $\rho \in (\rho(A), 1)$, \eqref{eq:spectral_decay} states that for all $k \geq 1$,
\begin{align*}
  \norm{A^{k-1}} \leq \hinfnorm{\Phi_A(\rho z)} \rho^k = \hinfnorm{ \rho^{-1} \Phi_{\rho^{-1} A}(z) }  \rho^k = \norm{\Phi_{\rho^{-1} A}} \rho^{k-1} \:.
\end{align*}
\end{proof}
Combining these last two claims with \eqref{eq:beta_coeff_markov}
and using the fact that $\E_{\nu_t}[\norm{X}^2]
\leq \E_{\nu_\infty}[\norm{X}^2]$ for all $t \geq 1$, we have the following
corollary, which is the main result of this section.
\begin{cor}
\label{cor:mixing_coeff}
Fix any $\rho \in (\rho(A), 1)$. For any $k \geq 1$ we have
\begin{align}
    \beta(k) \leq \frac{\hinfnorm{\Phi_{\rho^{-1} A}}}{2} \sqrt{\Tr(P_\infty) +  \frac{n}{1-\rho^2}} \rho^k \:.
\end{align}
\end{cor}

\section{Least-Squares Temporal Difference Learning}
\label{sec:lstd}

We turn our attention to the LSTD estimator. The goal of LSTD
is to compute the value function $V^\pi$ associated with a policy $\pi$ for an MDP.
This is an important primitive operation in many RL algorithms, such as policy-iteration.

Consider an MDP $\mathcal{M} = (\mathcal{S}, \mathcal{A}, p, \gamma, r)$,
where $\mathcal{S}$ denotes the state-space,
$\mathcal{A}$ denotes the action-space, $p : \mathcal{S} \times \mathcal{A} \longrightarrow \mu(S)$
denotes the transition kernel of the dynamics with $\mu(S)$ denoting the space of measures on $\mathcal{S}$,
$\gamma \in (0, 1)$ is the discount factor, and $r : \mathcal{S} \times \mathcal{A} \longrightarrow \R$
is the reward function.
Given a policy $\pi : \mathcal{S} \longrightarrow \mathcal{A}$, its value function $V^\pi : \mathcal{S} \longrightarrow \R$ is defined as
\begin{align*}
  V^\pi(x) := \E\left[ \sum_{k=0}^{\infty} \gamma^k r(X_k, \pi(X_k)) \; \bigg| \; X_0 = x \right] \:, \:\: X_{k+1} \sim p(\cdot | X_k, \pi(X_k) ) \:.
\end{align*}
Bellman's equation for the discounted, infinite-horizon
cost~\cite{bertsekas07} states that $V^\pi$ is the solution
to the fixed-point equation
\begin{align}
  V^\pi(x) = r(x, \pi(x)) + \gamma \E_{x' \sim p(\cdot|x,\pi(x))}[V^\pi(x')] \:, \:\: x \in \mathcal{S} \:. \label{eq:bellman_general}
\end{align}
When $\mathcal{S}$ is finite, dynamic programming can be used to solve
\eqref{eq:bellman_general}.
However, when $\mathcal{S}$ is continuous,
solving \eqref{eq:bellman_general} in general is difficult
without imposing additional structure.
By assuming that
$V^\pi$ admits the representation $V^\pi(x) = \ip{\phi(x)}{v_\pi}$ for some feature map $\phi : \mathcal{S} \longrightarrow \R^{d}$,
one turns \eqref{eq:bellman_general} into a system of
linear equations; this is known as the \emph{linear-architecture} assumption.
Specifically, if the dynamics $p(\cdot|x,u)$ are known, then
$V^\pi$ can be recovered as the solution to the system of linear equations
for $v_\pi$,
\begin{align}
  \ip{\phi(x) - \gamma \psi(x) }{v_\pi} = r(x, \pi(x)) \:, \:\: \psi(x) := \E_{x'\sim p(\cdot|x,\pi(x))}[ \phi(x') ] \:. \label{eq:bellman_linear}
\end{align}
Of course, we are interested in settings where the dynamics $p(\cdot|x,u)$ are not
known, and hence we cannot directly compute $\psi(x)$ in \eqref{eq:bellman_linear}.
This is where the LSTD estimator enters the picture:
given a trajectory $\{(X_k, R_k, X_{k+1})\}_{k=1}^{N}$ of length $N$,
the LSTD estimator $\widehat{v}_{\mathrm{lstd}}$ approximates the solution
to \eqref{eq:bellman_linear} by solving
\begin{align}
  \widehat{v}_{\mathrm{lstd}} = \left( \sum_{k=1}^{N} \phi(X_k)(\phi(X_k) - \gamma \phi(X_{k+1}))^\T \right)^{\dag} \left( \sum_{k=1}^{N} \phi(X_k) R_k \right) \:, \label{eq:lstd_estimator}
\end{align}
where $(\cdot)^{\dag}$ denotes the pseudo-inverse.
The curious looking nature of \eqref{eq:lstd_estimator}
accounts for the fact that when
$\phi(X_k) - \gamma \phi(X_{k+1})$ is used as an estimate
for $\phi(X_k) - \gamma \psi(X_k)$ in \eqref{eq:bellman_linear},
the noise in the linear measurement is not independent
from the covariate; see e.g.~\cite{bradtke96} for a more detailed discussion
of the issue. For completeness, in
Appendix~\ref{sec:appendix:lstd} we provide a more rigorous
justification for the estimator \eqref{eq:lstd_estimator} which
follows the development in Lazaric et al.~\cite{lazaric12}.

We will let the matrix $\Phi \in \R^{N \times d}$ denote the matrix
where the $k$-th row is $\phi(X_k)$.
While the main result of this section is a bound on the sample complexity of
the LSTD estimator on LQR, we first consider the implications of
Theorem~\ref{thm:mendelson_markov} on LSTD when both the features
$\phi$ and the rewards are bounded, in order to compare to the setting
of Lazaric et al.
We will then study the LQR problem, which is the simplest
non-trivial MDP which relaxes these boundedness assumptions.

\subsection{Bounded features and rewards}

For this section only we assume that $\sup_{x \in \mathcal{S}} \:\norm{\phi(x)}_\infty^2 \leq \overline{L}$
and $\sup_{x \in \mathcal{S}, a \in \mathcal{A}} \abs{r(s,a)} \leq R_{\mathrm{max}}$.
Under these assumptions, we immediately have $\sup_{x \in \mathcal{S}} \abs{V^\pi(x)} \leq \frac{1}{1-\gamma} R_{\mathrm{max}} := V_{\mathrm{max}}$.
The following result from Lazaric et al.\ gives a
bound on the in-sample prediction error of the estimator
$\widehat{V^\pi}(\cdot) := \ip{\phi(\cdot)}{\widehat{v}_{\mathrm{lstd}}}$.
\begin{thm}[Theorem 1, Lazaric et al.~\cite{lazaric12}]
\label{thm:lstd_bound}
With probability at least $1-\delta$, we have
\begin{align}
  \norm{ \widehat{V^\pi} - V^\pi }_N \leq \frac{\gamma}{1-\gamma} V_{\mathrm{max}} \sqrt{\frac{\overline{L} d}{\nu_N}} \left( \sqrt{\frac{8 \log(2d/\delta)}{N}} + \frac{1}{N} \right) \:,
\end{align}
where $\nu_N$ is the smallest non-zero eigenvalue of $\frac{1}{N} \Phi^\T \Phi$
and $\norm{\cdot}_N$ denotes the $L^2$-norm w.r.t. the empirical measure $\frac{1}{N} \sum_{k=1}^{N} \delta_{X_k}$.
\end{thm}
Immediately, Corollary~\ref{cor:mendelson_markov} combined with
Theorem~\ref{thm:lstd_bound} yield the following corollary.
\begin{cor}
Suppose that the stochastic process $\{ \phi(X_k) \}_{k=1}^{\infty}$ mixes
to some stationary measure $\nu_\infty$ at a rate $\beta(k) \leq \Gamma \rho^k$.
Furthermore, suppose that
\begin{align}
  0 < \ell \leq \lambda_{\min}(\E_{\nu_\infty}[\phi(X)\phi(X)^\T]) \leq L \:, \:\: \sup_{t \in S^{d-1}} \frac{\norm{\ip{\phi(X)}{t}}_{L^2(\nu_\infty)}}{\norm{\ip{\phi(X)}{t}}_{L^1(\nu_\infty)}} \leq O(1) \:. \label{eq:lstd_contractivity}
\end{align}
Fix a $\delta \in (0, 1)$, and suppose that $N$ satisfies
\begin{align*}
  \frac{N}{\log(\Gamma N/\delta)\log\log(\Gamma N/\delta)} \geq \Omega\left(\frac{1}{1-\rho} \frac{dL}{\ell}\right) \:.
\end{align*}
Then, with probability at least $1-\delta$,
\begin{align*}
  \norm{ \widehat{V^\pi} - V^\pi }_N \leq O\left(\frac{\gamma}{1-\gamma} V_{\mathrm{max}} \sqrt{\frac{\overline{L} d}{\ell}} \left( \sqrt{\frac{\log(d/\delta)}{N}} + \frac{1}{N} \right) \right) \:.
\end{align*}
\end{cor}
We remark that Lemma 4 of Lazaric et al.\ also provides an analysis of
$\lambda_{\min}(\frac{1}{N} \Phi^\T \Phi)$, but under the boundedness
assumptions of this section. Let us compare Corollary~\ref{cor:mendelson_markov}
to their Lemma 4.
Specializing their result to the case when the mixing is characterized by $\beta(k) \leq (1/2)^k$,
they prove that
$\lambda_{\min}\left(\frac{1}{N} \Phi^\T \Phi\right) \geq \Omega(\ell)$
where $\ell = \lambda_{\min}(\E_{\nu_\infty}[\phi(X) \phi(X)^\T])$
as long as
\begin{align*}
  \frac{N}{\log^2(N/\delta)} \geq \Omega\left( \frac{\overline{L} d^2}{\ell} \right) \:.
\end{align*}
Under the same setting,
as long as the contractivity condition \eqref{eq:lstd_contractivity} holds for
the stationary distribution,
our result relaxes the condition on $N$ to
\begin{align*}
  \frac{N}{\log(N/\delta) \log\log(N/\delta)} \geq \Omega\left( \frac{ L d }{\ell} \right) \:,
\end{align*}
where $L = \lambda_{\max}(\E_{\nu_\infty}[\phi(X) \phi(X)^\T])$.
Our work thus improves on the bound from Lazaric et al. by
reducing the minimum trajectory length from
$N \geq \widetilde{\Omega}(d^2)$ to $N \geq \widetilde{\Omega}(d)$.

\subsection{Linear Quadratic Regulator}
\label{sec:lstd:lqr}

We now study the performance of LSTD on LQR.
The LQR problem is an MDP with linear dynamics
\begin{align}
  X_{k+1} = A X_k + B U_k + W_k \:, \:\: W_k \sim \mathcal{N}(0, I) \:, \label{eq:linear_dynamcis_with_inputs}
\end{align}
and quadratic rewards
\begin{align*}
  r(x, u) = -(x^\T Q x + u^\T R u) \:,
\end{align*}
where $A$ is $n \times n$, $B$ is $n \times n_i$,
$Q$ and $R$ are positive-definite matrices, and $W_k$ is independent from $W_{k'}$ for all $k \neq k'$.
It is well known that the LQR problem can be solved with a linear
feedback policy $\pi(x) = K x$, and hence we will assume linear policies
in the sequel. We will further assume that the policy $\pi$ \emph{stabilizes}
the dynamics, i.e. the closed-loop matrix $L := A + BK$ is a stable matrix.
This stability assumption ensures that
the dynamics mix and the value function is finite.
We note that our analysis does not handle the case when
$L$ is not stable, but $\sqrt{\gamma} L$ is. In this case, the
value function is finite, but the dynamics do not mix.

Under our assumptions, it is straightforward to show by Bellman's equation \eqref{eq:bellman_general}
that $V^\pi(x) = -x^\T P_\pi x - \eta \Tr(P_\pi)$,
where $\eta := \gamma/(1-\gamma)$ and $P_\pi$ uniquely solves the discrete-time Lyapunov equation,
\begin{align*}
  (\gamma^{1/2} L)^\T P_\pi (\gamma^{1/2}L) - P_\pi + (Q + K^\T R K) = 0 \:.
\end{align*}
Furthermore, the stationary distribution of the dynamics
is $\nu_\infty = \mathcal{N}(0, P_\infty)$, where $P_\infty$ uniquely solves the
Lyapunov equation $L P_\infty L^\T - P_\infty + I = 0$.
To cast this problem into the linear-architecture format of LSTD,
we define the feature map $\phi(x)$ as $\phi(x) = \svec(xx^\T + \eta I)$.
Here, $\svec : \mathrm{Sym}_{n \times n} \longrightarrow \R^{n(n+1)/2}$
is the linear operator mapping the space of $n \times n$ symmetric matrices
(denoted $\mathrm{Sym}_{n \times n}$)
to vectors while preserving the property that
$\ip{\svec(M_1)}{\svec(M_2)}_{\R^{n(n+1)/2}} = \ip{M_1}{M_2}_{\mathrm{Sym}_{n \times n}}$ for all symmetric $M_1, M_2$. We will also let $\smat : \R^{n(n+1)/2} \longrightarrow \mathrm{Sym}_{n \times n}$ denote the inverse of $\svec$.
Hence in our setting, $d$ (the dimension of the lifted features)
is $d= n(n+1)/2$.
We will denote $v_\pi = \svec(P_\pi)$.

The main result of this section is the following
theorem which gives a bound on the error of the
difference between the LSTD estimator $\widehat{P} = \smat(\widehat{v}_{\mathrm{lstd}})$
and the true value function $P_\pi$.
\begin{thm}
\label{thm:lstd_estimate}
Fix $\delta \in (0, 1)$ and $\rho \in (\rho(L), 1)$.
Define $\widetilde{\Gamma} := \hinfnorm{\Phi_{\rho^{-1}L}}\sqrt{\Tr(P_\infty) + n/(1-\rho^2)}$.
Let $\widehat{P}$ denote the LSTD estimator \eqref{eq:lstd_estimator}
for the LQR problem.
Suppose that $N$ is large enough to satisfy
\begin{align}
  \frac{N}{\log(\widetilde{\Gamma} N/\delta)\log\log(\widetilde{\Gamma} N/\delta)} \geq
  \Omega\left( \frac{\max\{ \Tr(P_\infty)^2, \eta^2 n \} }{(1-\rho) \lambda_{\min}^2(P_\infty)}  \right) \:. \label{eq:lstsq_N_requirement}
\end{align}
Then, with probability at least $1-\delta$,
\begin{align}
  \frac{\norm{ \widehat{P} - P_\pi }_F}{\norm{P_\pi}_F} \leq O\left( \frac{\eta \sqrt{\norm{P_\infty}} \max\{\Tr(P_\infty), \eta\sqrt{n}\}}{\sqrt{N} \lambda_{\min}^2(P_\infty)}   \mathrm{polylog}(N, n, 1/\delta) \right) \:. \label{eq:lstd_relerr_bound}
\end{align}
\end{thm}
Before we prove Theorem~\ref{thm:lstd_estimate}, we make several remarks
on the behavior of \eqref{eq:lstd_relerr_bound}.
Let us first simplify it to ease the exposition, by applying the bound
$\Tr(P_\infty) \leq n \norm{P_\infty}$ and
assuming we are in the regime when $n \gg (\eta / \norm{P_\infty})^2$ so that
$n \norm{P_\infty}$ dominates $\eta\sqrt{n}$.
With these simplifications, \eqref{eq:lstd_relerr_bound} becomes
\begin{align*}
  \frac{\norm{\widehat{P} - P_\pi}_F}{\norm{P_\pi}_F} \leq O\left( \frac{n}{(1-\gamma)\sqrt{N}} \frac{\norm{P_\infty}^{3/2}}{\lambda_{\min}^2(P_\infty)}  \mathrm{polylog}(N, n, 1/\delta) \right) \:,
\end{align*}
which yields the sufficient condition that
\begin{align}
  N \geq \widetilde{\Omega}\left( \frac{n^2}{(1-\gamma)^2 \varepsilon^2} \frac{\kappa^3(P_\infty)}{\lambda_{\min}(P_\infty)} \right) \:, \:\: \kappa(P_\infty) := \frac{\norm{P_\infty}}{\lambda_{\min}(P_\infty)} \label{eq:N_simple}
\end{align}
samples ensure the relative error is less than $\varepsilon$.

We now remark on the dependence of \eqref{eq:N_simple} on the spectral properties of $P_\infty$.
In particular, \eqref{eq:N_simple} suggests that as $\kappa(P_\infty)$
increases, more samples are needed to reach a fixed $\varepsilon$ tolerance.
In controls parlance, the matrix $P_\infty$ is known as the controllability gramian.
A system with large $\kappa(P_\infty)$ is one where different modes exhibit
qualitatively different behaviors. The simplest example of this is when
the closed-loop matrix is $L = \diag(\rho_1, ..., \rho_n)$ with $\rho_k \in (0, 1)$,
in which case $P_\infty = \diag(1/(1-\rho_1^2), ..., 1/(1-\rho_n^2))$.
Here, as $\rho_1$ increases towards one, \eqref{eq:N_simple} predicts that
estimating the value function requires more samples.
In Section~\ref{sec:experiments:synthetic}, we show that this predicted behavior
actually occurs in numerical simulations.

Let us now compare \eqref{eq:N_simple} to the setting of Dean et al.~\cite{dean17},
where ordinary least-squares is used to estimate the
state-transition matrices $(A,B)$ of \eqref{eq:linear_dynamcis_with_inputs}, and
a robust control procedure is used to design a controller to stabilize
\eqref{eq:linear_dynamcis_with_inputs}. Ignoring problem specific parameters,
Corollary 4.3 of Dean et al.\ states that at most $\widetilde{\Omega}( n/\varepsilon^2 )$
samples are needed to design a controller which incurs a relative error of at most $\varepsilon$.
On the other hand, \eqref{eq:N_simple} suggests that $\widetilde{\Omega}(n^2/\varepsilon^2)$
samples are needed to estimate a single value function.
This gap between upper bounds suggests that for LQR, model-based methods may
perform better than policy-iteration methods such as Least-Squares Policy
Iteration (LSPI), which require multiple policy evaluation steps.
In Section~\ref{sec:experiments:lspi}, we provide empirical evidence that shows this is indeed the case
for certain LQR instances.
We leave as
future work lower bounds to separate the sample complexities of
model-free and model-based methods for LQR.

The remainder of the section is dedicated to the proof
of Theorem~\ref{thm:lstd_estimate}.
Because the estimator \eqref{eq:lstd_estimator} is not a standard
least-squares estimator (despite its name), some analysis is needed
to manipulate the estimator into a form that is easier to analyze.
We follow the development in Lazaric et al.\ and
state the main structural result of their paper below.
For completeness,
we provide a proof in Appendix~\ref{sec:appendix:lstd},
noting that our development removes the technical restriction that the last state observed
along the trajectory is discarded.
\begin{lem}[Lazaric et al.~\cite{lazaric12}]
\label{lemma:structural_bound_lstd_lqr}
As long as $\Phi$ has full column rank,
the LSTD estimator $\widehat{P}$ satisfies the following inequality,
\begin{align}
  \norm{\widehat{P} - P_\pi}_F \leq \frac{\eta \bignorm{ \sum_{k=1}^{N} \phi(X_k)(\phi(X_{k+1}) - \E[ \phi(X_{k+1}) | X_k ])^\T v_\pi }}{ \lambda_{\min}\left(\sum_{k=1}^{N} \phi(X_k) \phi(X_k)^\T\right)  } \:. \label{eq:lstd_error_decomposition}
\end{align}
\end{lem}
The proof of Theorem~\ref{thm:lstd_estimate} proceeds by bounding
the terms on the RHS of \eqref{eq:lstd_error_decomposition}.
Theorem~\ref{thm:mendelson_markov} from Section~\ref{sec:bound}
combined with the mixing analysis in Section~\ref{sec:mixing_lds}
can be directly applied to estimate
the minimum eigenvalue of the matrix $\sum_{k=1}^{N} \phi(X_k) \phi(X_k)^\T$.
The term in the numerator can also be dealt with via standard martingale techniques.
We start with the eigenvalue bound.
\begin{lem}
\label{lemma:eigenvalue_bound_lstd_lqr}
Suppose that the hypothesis of Theorem~\ref{thm:lstd_estimate} hold. Then, with probability at least $1-\delta$,
\begin{align*}
  \lambda_{\min}\left(\sum_{k=1}^{N} \phi(X_k) \phi(X_k)^\T\right) \geq \Omega(N \lambda_{\min}^2(P_\infty)) \:.
\end{align*}
\end{lem}
Before we prove Lemma~\ref{lemma:eigenvalue_bound_lstd_lqr}, we state a technical result
that we will need.
\begin{lem}
\label{lem:gaussian_small_ball}
Let $f$ be a degree $d$ polynomial and $x \sim \mathcal{N}(0, I)$.
We have that
\begin{align*}
  \Pr\{ \abs{f(x)} \geq \norm{f}_{L^2} / \sqrt{2} \} \geq \frac{1}{4 \times 3^{2d}} \:.
\end{align*}
\end{lem}
\begin{proof}
By the Paley-Zygmund inequality, for any $\theta \in (0, 1)$,
\begin{align*}
  \Pr\{ \abs{f(x)} \geq \sqrt{\theta} \norm{f}_{L^2} \} = \Pr\{ f(x)^2 \geq \theta \norm{f}_{L^2}^2 \} \geq (1-\theta)^2 \frac{ \norm{f}_{L^2}^4 }{ \norm{f}_{L^4}^4 } \:.
\end{align*}
Now by Gaussian hypercontractivity (Lemma~\ref{lem:gaussian_hypercontractivity}), we have that
$\norm{f}_{L^4} \leq 3^{d/2} \norm{f}_{L^2}$.
The claim follows by setting $\theta=1/2$ and plugging in this inequality.
\end{proof}

\begin{proof}
(Lemma~\ref{lemma:eigenvalue_bound_lstd_lqr}).
We first need to estimate the small-ball probability \eqref{eq:small_ball}.
To do this, write $X = P_\infty^{1/2} g$, with $g \sim \mathcal{N}(0, I)$.
Fix any $v \in S^{d-1}$, and let $V = \smat(v)$.
With this notation,
\begin{align*}
  f_v(g) := \ip{v}{\phi(P_\infty^{1/2} g)} = \ip{V}{P_\infty^{1/2} gg^\T P_\infty^{1/2} + \eta I}
  = g^\T P_\infty^{1/2} V P_\infty^{1/2} g + \eta \Tr(V) \:.
\end{align*}
Clearly $f_v$ is a degree two polynomial in $g$. Furthermore,
using Proposition~\ref{prop:gaussian_fourth_moment},
we can lower bound its second moment by
\begin{align*}
  \E[f_v(g)^2] &= \E[ (g^\T P_\infty^{1/2} V P_\infty^{1/2} g)^2 ] + \eta^2 \Tr(V)^2 + 2 \eta \E[g^\T P_\infty^{1/2} T P_\infty^{1/2} g] \Tr(V) \\
  &= 2\norm{P_\infty^{1/2} V P_\infty^{1/2}}_F^2 + \ip{V}{P_\infty}^2 + \eta^2 \Tr(V)^2 + 2\eta\ip{T}{P_\infty} \Tr(V) \\
  &= 2\norm{P_\infty^{1/2} V P_\infty^{1/2}}_F^2 + ( \ip{V}{P_\infty} + \eta \Tr(V) )^2 \\
  &\geq 2\norm{P_\infty^{1/2} V P_\infty^{1/2}}_F^2 \geq 2 \lambda_{\min}^2(P_\infty) \:.
\end{align*}
The last inequality follows by standard properties of the Kronecker product,
\begin{align*}
  \norm{P_\infty^{1/2} V P_\infty^{1/2}}_F^2 &= \norm{\vec(P_\infty^{1/2} V P_\infty^{1/2})}^2
  = \norm{ (P_\infty^{1/2} \otimes P_\infty^{1/2}) \vec(V) }^2 \\
  &\geq \sigma^2_{\min}(P_\infty^{1/2} \otimes P_\infty^{1/2}) \norm{\vec(V)}^2
  =\sigma^2_{\min}(P_\infty^{1/2} \otimes P_\infty^{1/2}) \norm{V}_F^2 \\
  &=\sigma^2_{\min}(P_\infty^{1/2} \otimes P_\infty^{1/2})
  = \lambda^2_{\min}(P_\infty) \:.
\end{align*}
Combining these inequalities, we have that
$\norm{f_v}_{L^2} \geq \sqrt{2} \lambda_{\min}(P_\infty)$.
By Lemma~\ref{lem:gaussian_small_ball},
we conclude
\begin{align*}
  \Pr_{\nu_\infty}\{ \abs{\ip{v}{\phi(X)}} \geq \lambda_{\min}(P_\infty) \} \geq 1/324 \:.
\end{align*}
Hence, we can take $\tau = \lambda_{\min}(P_\infty)$.
We now compute the second-moment $\E_{\nu_\infty}[ \norm{\phi(X)}^2 ]$,
\begin{align*}
  \E_{\nu_\infty}[ \norm{\phi(X)}^2 ] &= \E_{\nu_\infty}[\norm{X}^4] + 2 \eta \E_{\nu_\infty}[\norm{X}^2] + \eta^2 n \\
  &= 2 \norm{P_\infty}_F^2 + \Tr(P_\infty)^2 + 2 \eta \Tr(P_\infty) + \eta^2 n \\
  &\leq 3 \Tr(P_\infty)^2 + 2 \eta \Tr(P_\infty) + \eta^2 n \\
  &\leq 4 \Tr(P_\infty)^2 + \eta^2 (n + 1) \:.
\end{align*}
The claim now follows from the mixing time calculation in
Corollary~\ref{cor:mixing_coeff} and Theorem~\ref{thm:mendelson_markov}.
In order to apply Corollary~\ref{cor:mixing_coeff} to the stochastic
process $\{\phi(X_k)\}_{k=1}^{\infty}$,
we let $\phi^{-1}(A) = \{ x \in \R^n : \phi(x) \in A \}$ for any measurable set $A \subseteq \R^d$ and
observe that for any $k \geq 1$,
\begin{align*}
    \abs{\Pr_{X_k}(\phi(X_k) \in A|X_0{=}x) - \nu_\infty(\phi(X_k) \in A)}
    &= \abs{\Pr_{X_k}(X_k \in \phi^{-1}(A)|X_0{=}x) - \nu_\infty(X_k \in \phi^{-1}(A))} \\
    &\leq \norm{ \Pr_{X_k}(\cdot | X_0{=}x) - \nu_\infty}_{\mathrm{tv}} \:,
\end{align*}
and hence $\norm{\Pr_{\phi(X_k)}(\cdot|X_0{=}x) - \nu_\infty(\phi(X_k) \in \cdot)}_{\mathrm{tv}} \leq\norm{ \Pr_{X_k}(\cdot | X_0{=}x) - \nu_\infty}_{\mathrm{tv}}$.
\end{proof}
We now turn our attention to the term in the numerator of \eqref{eq:lstd_error_decomposition}.
In the sequel, we will make repeated use of the Hanson-Wright inequality from
Rudelson and Vershynin~\cite{rudelson13}.
We first start with a technical claim that will be used in the main proof.
\begin{prop}
\label{prop:hanson_wright_xy}
Let $x \sim \mathcal{N}(0, \Sigma_1)$ and $y \sim \mathcal{N}(0, \Sigma_2)$ be independent.
Fix a matrix $M$.
There exists a universal constant $c$ such that
with probability at least $1-\delta$,
\begin{align*}
  \abs{x^\T M y} \leq \sqrt{ \frac{\norm{ \Sigma_1^{1/2} M \Sigma_2^{1/2} }_F^2}{c} \log(2/\delta) } + \frac{\norm{\Sigma_1^{1/2} M \Sigma_2^{1/2}}}{c} \log(2/\delta) \:.
\end{align*}
\end{prop}
\begin{proof}
Observe we can write
\begin{align*}
  x^\T M y = w^\T \Sigma_1^{1/2} M \Sigma_2^{1/2} v = \frac{1}{2} \begin{bmatrix} w \\ v \end{bmatrix}^\T \begin{bmatrix} 0 & \Sigma_1^{1/2} M \Sigma_2^{1/2} \\ \Sigma_2^{1/2} M^\T \Sigma_1^{1/2} & 0 \end{bmatrix} \begin{bmatrix} w \\ v \end{bmatrix} \:,
\end{align*}
where $(w,v)$ is an isotropic Gaussian.
Now recall that $\bignorm{ \begin{bmatrix} 0 & A \\ A^\T & 0 \end{bmatrix} } = \norm{A}$.
The result now follows by the Hanson-Wright inequality.
\end{proof}
We now establish a bound on the numerator of \eqref{eq:lstd_error_decomposition}.
\begin{lem}
\label{lemma:upstairs_bound_lstd_lqr}
Fix $\delta \in (0, 1)$. With probability at least $1-\delta$, we have,
\begin{align*}
  &\bignorm{\sum_{k=1}^{N} \phi(X_k)(\phi(X_{k+1}) - \E[ \phi(X_{k+1}) | X_k ])^\T v_\pi } \\
  &\qquad\leq O( \norm{v_\pi} \sqrt{N} (\Tr(P_N) + \eta \sqrt{n})\norm{L P_N^{1/2}}  \mathrm{polylog}(N, n, 1/\delta) ) \:.
\end{align*}
\end{lem}
\begin{proof}
The proof uses a standard martingale argument. The only complication here is
that the random vectors are heavy tailed, and hence a
truncation argument is needed; we use a truncation argument
similar to Lemma 1 of Bubeck et al.~\cite{bubeck13}.
In the proof, constants $c_i$ will denote universal constants.

We introduce the shorthand notation $\phi_k := \phi(X_k)$
and $\Delta_k := \ip{\phi_{k+1} - \E[\phi_{k+1} | \F_k]}{v_\pi}$.
We first show the following identities,
\begin{align*}
  \Delta_k &= 2 W_k^\T P_\pi L X_k + W_k^\T P_\pi W_k - \Tr(P_\pi) \:, \\
  \E[ \Delta_k^2 | \F_k ] &= 4 \norm{P_\pi L X_k}^2 + 2 \norm{P_\pi}_F^2 \:.
\end{align*}
Observe that by the linearity of $M \mapsto \svec(M)$,
\begin{align*}
  \phi_{k+1} - \E[\phi_{k+1}|\F_k] &= \svec( X_{k+1}X_{k+1}^\T + \eta I ) - \E[ \svec( X_{k+1}X_{k+1}^\T + \eta I ) | \F_k] \\
  &= \svec\left( X_{k+1}X_{k+1}^\T + \eta I - \E[ X_{k+1}X_{k+1}^\T + \eta I | \F_k] \right) \\
  &= \svec\left( X_{k+1}X_{k+1}^\T - \E[ X_{k+1}X_{k+1}^\T | \F_k ] \right) \\
  &= \svec\left( (L X_k + W_k)(L X_k + W_k)^\T - \E[ (L X_k + W_k)(L X_k + W_k)^\T | \F_k ] \right) \\
  &= \svec\left( LX_kX_k^\T L^\T + L X_k W_k^\T + W_k X_k^\T L^\T + W_kW_k^\T - LX_kX_k^\T L^\T - I \right) \\
  &= \svec\left( LX_k W_k^\T + W_kX_k^\T L^\T + W_kW_k^\T - I \right) \:.
\end{align*}
Therefore,
\begin{align*}
  \Delta_k &= \ip{\phi_{k+1} - \E[\phi_{k+1}|\F_k]}{v_\pi} \\
  &= \ip{LX_k W_k^\T + W_kX_k^\T L^\T + W_kW_k^\T - I}{P_\pi} \\
  &= \ip{LX_kW_k^\T}{P_\pi} + \ip{W_kX_k^\T L^\T}{P_\pi} + W_k^\T P_\pi W_k - \Tr(P_\pi) \\
  &= 2W_k^\T P_\pi L X_k + W_k^\T P_\pi W_k - \Tr(P_\pi) \:.
\end{align*}
Next, expanding out $\E[\Delta_k^2 | \F_k]$,
\begin{align*}
  \E[ \Delta_k^2 | \F_k ] = \E[4 (W_k^\T P_\pi L X_k)^2 + (W_k^\T P_\pi W_k - \Tr(P_\pi))^2 + 4 W_k^\T P_\pi L X _k (W_k^\T P_\pi W_k - \Tr(P_\pi))|\F_k] \:.
\end{align*}
The claimed identities now follows by observing that,
\begin{align*}
  \E[ (W_k^\T P_\pi L X_k)^2 | \F_k] &= \E[ W_k^\T P_\pi L X_kX_k^\T L^\T P_\pi W_k | \F_k] \\
  &= \Tr( P_\pi L X_k X_k^\T L^\T P _\pi) = \norm{ P_\pi L X_k }^2 \:, \\
  \E[(W_k^\T P_\pi W_k - \Tr(P_\pi))^2] &= \mathrm{Var}(W_k^\T P_\pi W_k) \\
  &\stackrel{(a)}{=} \E[(W_k P_\pi W_k)^2] - \Tr(P_\pi)^2 = 2\norm{P_\pi}_F^2 \:, \\
  \E[ W_k^\T P_\pi L X_k(W_k^\T P_\pi W_k - \Tr(P_\pi)) | \F_k] &= 0 \:,
\end{align*}
where (a) follows from Proposition~\ref{prop:gaussian_fourth_moment}.

By the Hanson-Wright inequality, with probability at least $1-\delta$,
\begin{align*}
  \abs{W_k^\T P_\pi W_k - \Tr(P_\pi)} \leq \sqrt{ \frac{\norm{P_\pi}_F^2}{c_1} \log(2/\delta) } + \frac{\norm{P_\pi}}{c_1} \log(2/\delta) \:.
\end{align*}
Also, since $X_k \sim N(0, P_k)$ and is independent of $W_k$, by Proposition~\ref{prop:hanson_wright_xy},
with probability at least $1-\delta$,
\begin{align*}
  \abs{2 W_k^\T P_\pi L X_k} \leq \sqrt{ \frac{\norm{P_\pi L P_k^{1/2}}_F^2}{c_2} \log(2/\delta) } + \frac{\norm{P_\pi L P_k^{1/2}}}{c_2} \log(2/\delta) \:.
\end{align*}
Hence, applying the triangle inequality and a union bound,
there is an event $\calE_{\Delta,\delta}$ such that $\Pr(\calE_{\Delta,\delta}^c) \leq \delta$ and on $\calE_{\Delta,\delta}$,
\begin{align*}
  \max_{1 \leq k \leq N} \abs{\Delta_k} &\leq
    \sqrt{ \frac{\norm{P_\pi}_F^2}{c_1} \log(4N/\delta) } + \frac{\norm{P_\pi}}{c_1} \log(4N/\delta) \\
    &\qquad+
    \sqrt{ \frac{\norm{P_\pi L P_N^{1/2}}_F^2}{c_2} \log(4N/\delta) } + \frac{\norm{P_\pi L P_N^{1/2}}}{c_2} \log(4N/\delta) \\
    &:= M_{\Delta,\delta} \:.
\end{align*}
Next, by standard Gaussian concentration results and
a union bound, combined with the fact that $P_k \preccurlyeq P_N$ for all $1 \leq k \leq N$,
there exists an event $\calE_{X,\delta}$, such that
$\Pr(\calE_{X,\delta}^c) \leq \delta$ and on $\calE_{X,\delta}$,
\begin{align*}
  \max_{1 \leq k \leq N} \norm{X_k} \leq \sqrt{\Tr(P_N)} + \norm{P_N}^{1/2} \sqrt{2 \log(N/\delta)} := M_{X,\delta} \:.
\end{align*}
Furthermore, since $P_\pi L X_k \sim \mathcal{N}(0, P_\pi L P_k L^\T P_\pi)$,
similar arguments yield that
there is an event $\calE_{PLX,\delta}$ with $\Pr(\calE_{PLX,\delta}) \leq \delta$ and on
$\calE_{PLX,\delta}$,
\begin{align*}
  \max_{1 \leq k \leq N} \norm{P_\pi L X_k} \leq \norm{P_\pi L P_N^{1/2}}_F + \norm{P_\pi L P_N^{1/2}} \sqrt{2\log(N/\delta)} := M_{PLX,\delta} \:.
\end{align*}
Hence, on $\calE_{X,\delta} \cap \calE_{PLX,\delta}$, setting $M_k := \sqrt{k} M_{\Delta,\delta}$,
\begin{align*}
  T_1 &:= \bignorm{ \sum_{k=1}^{N} \phi_k \E[ \Delta_k \ind_{\abs{\Delta_k} > M_k} | \F_k ] }
      \leq \sum_{k=1}^{N} \norm{\phi_k} \E[ \abs{\Delta_k} \ind_{\abs{\Delta_k} > M_k} | \F_k ] \\
      &= \sum_{k=1}^{N} \norm{\phi_k} \E\left[ \frac{\Delta_k^2}{\abs{\Delta_k}} \ind_{\abs{\Delta_k} > M_k} | \F_k \right]
      \leq \sum_{k=1}^{N} \norm{\phi_k} \frac{\E[ \Delta_k^2 | \F_k ]}{M_k} \\
      &\leq 4\sum_{k=1}^{N} (\norm{X_k}^2 + \eta \sqrt{n}) \frac{\norm{P_\pi L X_k}^2 + \norm{P_\pi}_F^2}{M_k}
      \leq 4\sum_{k=1}^{N} (M_{X,\delta}^2 + \eta \sqrt{n}) \frac{M_{PLX,\delta}^2 + \norm{P_\pi}_F^2}{M_k} \\
      &= 4 (M_{X,\delta}^2 + \eta \sqrt{n}) \frac{M_{PLX,\delta}^2 + \norm{P_\pi}_F^2}{M_{\Delta,\delta} } \sum_{k=1}^{N} \frac{1}{\sqrt{k}}
      \leq 8 (M_{X,\delta}^2 + \eta \sqrt{n}) \frac{M_{PLX,\delta}^2 + \norm{P_\pi}_F^2}{M_{\Delta,\delta} } \sqrt{N} \:.
\end{align*}
Now, observe that,
\begin{align*}
  \frac{M_{PLX,\delta}^2 + \norm{P_\pi}_F^2}{M_{\Delta,\delta} } &\leq \frac{ 2\norm{P_\pi L P_N^{1/2}}_F^2 + 4 \norm{P_\pi L P_N^{1/2}}^2 \log(N/\delta) + \norm{P_\pi}_F^2 }{ M_{\Delta,\delta} } \\
  &\leq c_3 (M_{PLX,\delta} + \norm{P_\pi}_F) \:.
\end{align*}
Hence,
\begin{align*}
  T_1 \leq c_4 (M_{X,\delta}^2 + \eta\sqrt{n}) (M_{PLX,\delta} + \norm{P_\pi}_F) \sqrt{N} \:.
\end{align*}
We now turn our attention to bounding the martingale difference sequence,
\begin{align*}
  T_2 := \bignorm{ \sum_{k=1}^{N} \phi_k\ind_{\norm{X_k} \leq M_{X,\delta}} (\Delta_k\ind_{\abs{\Delta_k} \leq M_k} - \E[ \Delta_k\ind_{\abs{\Delta_k} \leq M_k} | \F_k ])  } := \bignorm{ \sum_{k=1}^{N} Y_k } \:.
\end{align*}
Observe that $Y_k$ is $\F_{k+1}$-measurable and $\E[Y_k|\F_k] = 0$.
Also,
\begin{align*}
  \norm{Y_k} \leq 2(M_{X,\delta}^2 + \eta\sqrt{n})M_k = 2(M_{X,\delta}^2 + \eta\sqrt{n})M_{\Delta,\delta} \sqrt{k}
\end{align*}
holds almost surely.
Hence, by Freedman's inequality for matrix martingales (Corollary 1.3 of Tropp~\cite{tropp11}),
for every $\sigma > 0$, there exists
an event with probability at least $1-\delta$ such that
\begin{enumerate}[(a)]
  \item $\bignorm{\sum_{k=1}^{N} Y_k } \leq \frac{8}{3}(M_{X,\delta}^2 + \eta\sqrt{n}) M_{\Delta,\delta} \sqrt{N} + \sigma \sqrt{2 \log(d/\delta)}$, or
  \item $\sum_{k=1}^{N} \E[ \norm{Y_k}^2 | \F_{k} ] > \sigma^2$.
\end{enumerate}
We now bound $\sum_{k=1}^{N} \E[ \norm{Y_k}^2 | \F_{k} ]$ from above so we can exclude
the possibility of the second condition.
To do this, we observe that on $\calE_{X,\delta} \cap \calE_{PLX,\delta}$,
\begin{align*}
  \E[\norm{Y_k}^2 | \F_k] &\leq (M_{X,\delta}^2 + \eta\sqrt{n})^2 \E[ (\Delta_k\ind_{\abs{\Delta_k} \leq M_k} - \E[ \Delta_k\ind_{\abs{\Delta_k} \leq M_k} | \F_k ])^2 | \F_k] \\
  &\leq (M_{X,\delta}^2 + \eta\sqrt{n})^2 \E[ \Delta_k^2 \ind_{\abs{\Delta_k} \leq M_k} | \F_k ] \\
  &\leq (M_{X,\delta}^2 + \eta\sqrt{n})^2 \E[ \Delta_k^2 | \F_k ] \\
  &= (M_{X,\delta}^2 + \eta\sqrt{n})^2 ( 4 \norm{P_\pi L X_k}^2 + 2 \norm{P_\pi}_F^2 ) \\
  &\leq 4 (M_{X,\delta}^2 + \eta\sqrt{n})^2 (M_{PLX,\delta}^2 + \norm{P_\pi}_F^2) \:,
\end{align*}
from which we conclude that
$\sum_{k=1}^{N} \E[ \norm{Y_k}^2 | \F_{k} ] \leq 4 N (M_{X,\delta} + \eta\sqrt{n})^2 (M_{PLX,\delta}^2 + \norm{P_\pi}_F^2)$.
Hence, setting
\begin{align*}
  \sigma = 2\sqrt{N} (M_{X,\delta}^2 + \eta\sqrt{n})( M_{PLX,\delta} + \norm{P_\pi}_F  ) \:,
\end{align*}
we conclude there exists an event $\calE_{\mathrm{fr},\delta}$ such that
$\Pr(\calE_{\mathrm{fr},\delta}^c) \leq \delta$ and
on $\calE_{\mathrm{fr},\delta} \cap \calE_{X,\delta} \cap \calE_{PLX,\delta}$,
\begin{align*}
  T_2 \leq \frac{8}{3}(M_{X,\delta}^2 + \eta\sqrt{n}) M_{\Delta,\delta} \sqrt{N} + 2\sqrt{N} (M_{X,\delta}^2 + \eta\sqrt{n})( M_{PLX,\delta} + \norm{P_\pi}_F  ) \sqrt{2 \log(d/\delta)} \:.
\end{align*}
We are now ready to combine the above calculations.
We introduce the shorthand $\phitilde_k := \phi_k \ind_{\norm{X_k} \leq M_{X,\delta}}$
and $\Deltatilde_k := \Delta_k \ind_{\abs{\Delta_k} \leq M_k}$.
For what follows, we assume we are on the event
$\calE_{\Delta,\delta} \cap \calE_{\mathrm{fr},\delta} \cap \calE_{X,\delta} \cap \calE_{PLX,\delta}$.
By a union bound, this occurs with probability at least $1-4\delta$.
On this event, we have that $\phitilde_k = \phi_k$ and $\Deltatilde_k = \Delta_k$ for all $k=1, ..., N$.
Hence, using the fact that $\E[\Delta_k | \F_k] = 0$,
\begin{align*}
  \bignorm{ \sum_{k=1}^{N} \phi_k \Delta_k } &= \bignorm{ \sum_{k=1}^{N} \phitilde_k \Deltatilde_k } \\
  &\leq \bignorm{ \sum_{k=1}^{N} \phitilde_k \E[\Deltatilde_k | \F_k] } + \bignorm{  \sum_{k=1}^{N} \phitilde_k (\Deltatilde_k - \E[\Deltatilde_k | \F_k]) }  \\
  &= \bignorm{ \sum_{k=1}^{N} \phitilde_k (\E[\Deltatilde_k | \F_k] - \E[\Delta_k | \F_k])  } + \bignorm{  \sum_{k=1}^{N} \phitilde_k (\Deltatilde_k - \E[\Deltatilde_k | \F_k]) } \\
  &= \bignorm{ \sum_{k=1}^{N} \phitilde_k \E[\Delta_k \ind_{\abs{\Delta_k} > M_k} | \F_k]  } + \bignorm{  \sum_{k=1}^{N} \phitilde_k (\Deltatilde_k - \E[\Deltatilde_k | \F_k]) } \\
  &= T_1 + T_2 \\
  &\leq c_5 (M_{X,\delta}^2 + \eta\sqrt{n}) (M_{PLX,\delta} + \norm{P_\pi}_F) \sqrt{N} \\
  &\qquad + c_6 (M_{X,\delta}^2 + \eta\sqrt{n}) M_{\Delta,\delta} \sqrt{N} \\
  &\qquad + c_7 (M_{X,\delta}^2 + \eta\sqrt{n})( M_{PLX,\delta} + \norm{P_\pi}_F  ) \sqrt{\log(n/\delta)} \sqrt{N} \:.
\end{align*}
The claim now follows by observing that
\begin{align*}
  \norm{P_\pi L P_N^{1/2}}_F \leq \norm{P_\pi}_F \norm{L P_N^{1/2}} \:,
\end{align*}
followed by straightforward simplifications.
\end{proof}
Theorem~\ref{thm:lstd_estimate}
now readily follows by combining the
structural result of Lemma~\ref{lemma:structural_bound_lstd_lqr}
with the bounds established in
Lemma~\ref{lemma:eigenvalue_bound_lstd_lqr} and
Lemma~\ref{lemma:upstairs_bound_lstd_lqr},
using the fact that $P_N \preccurlyeq P_\infty$ and
$L P_\infty L^\T = P_\infty - I \preccurlyeq P_\infty$.

\section{Experiments}
\label{sec:experiments}

We conduct numerical experiments on LSTD for value function estimation, and
Least-Squares Policy Iteration (LSPI) for an end-to-end comparison with the
model-based methods in Dean et al.~\cite{dean17}.
Our implementation
is carried out in Python using \verb|numpy| for linear algebraic computations
and PyWren~\cite{jonas17} for parallelization.

In our first set of experiments, we construct synthetic examples where
we vary the condition number of the resulting closed-loop
controllability gramian matrix. We find that on these instances, as
the condition number increases, the required number of samples to estimate
the value function to fixed relative error increases, as
predicted by our result in Theorem~\ref{thm:lstd_estimate}.
In our second set of experiments, we compare model-free policy iteration (LSPI)
to two model-based methods:
(a) the na{\"i}ve nominal model controller which uses a controller
designed assuming that the nominal model has zero error, and (b)
a controller based on a semidefinite relaxation to the non-convex robust control problem
with static state-feedback.
Our experiments show that model-free policy iteration requires more samples
than model-based methods for the instances we consider.

\subsection{Synthetic Data}
\label{sec:experiments:synthetic}

The goal in this section is to showcase the qualitative behavior of LSTD on LQR
predicted by Theorem~\ref{thm:lstd_estimate} as the conditioning of the
closed-loop controllability gramian varies.
We consider several instances of LQR with $n=5$,
$Q=R=0.1 I_5$, and $\gamma=0.9$, where the state transition matrices $(A,B)$ and the policy $\pi(x) = Kx$
will be specified later.
For each configuration, we collect
$100$ trajectories of length $N=1000$. For each trajectory,
we take the first $N_p$ points for $N_p \in \{100, 200, ..., 1000\}$
and compute the LSTD estimator $\widehat{P}_{N_p}$ on the first $N_p$ data points.
We then compute the relative error
$\frac{\norm{P_\pi - \widehat{P}_{N_p}}_F}{\norm{P_\pi}_F}$ for each $N_p$,
and report the median and $25$-th to $75$-th percentile over the $100$ trajectories.

\begin{figure}[t!]
  \centering
  \begin{minipage}[t]{0.46\textwidth}
  \begin{center}
  \includegraphics[width=\columnwidth]{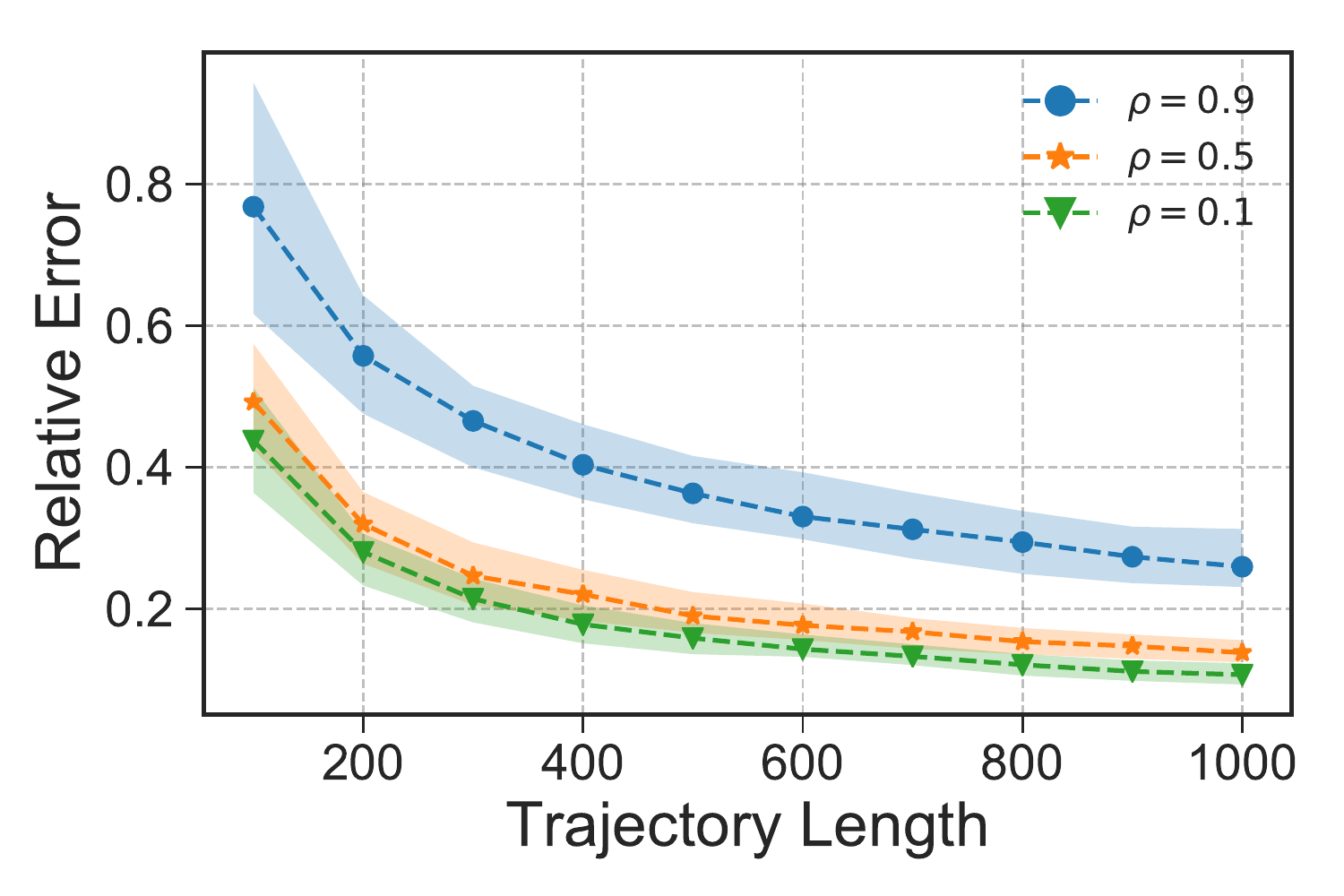}
  \end{center}
  \caption{
    Performance of LSTD on LQR instances where
    the closed loop response is $L=A+BK=\rho I_5$ for $\rho \in \{0.1,0.5,0.9\}$.
	The dashed line represents the median relative error, and the
	shaded region covers the $25$-th to $75$-th percentile
	of the relative error out of $100$ trajectories.
  }
    \label{fig:exp_simple}
  \end{minipage}
  \hspace{.02\textwidth}
    \begin{minipage}[t]{0.46\textwidth}
  \begin{center}
    \includegraphics[width=\columnwidth]{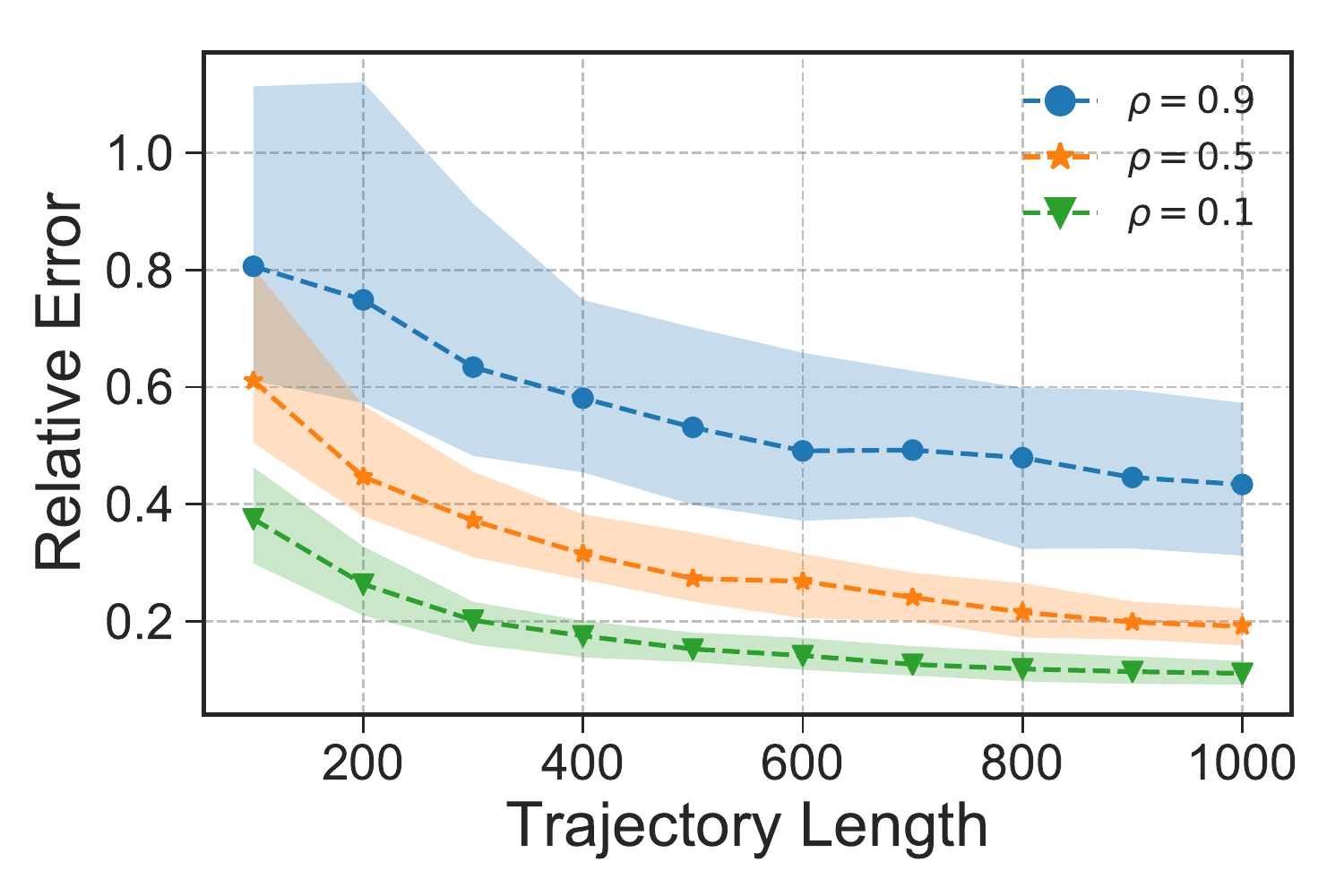}
  \end{center}
  \caption{
    Performance of LSTD on LQR instances where
    $B=K=0_{5 \times 5}$ and $A$ is generated randomly with spectral radius $\rho(A) \in \{0.1,0.5,0.9\}$.
	The dashed line represents the median relative error, and the
	shaded region covers the $25$-th to $75$-th percentile
	of the relative error out of $100$ trajectories.
  }
    \label{fig:exp_trace}
    \end{minipage}
\end{figure}

In the first experiment, we set $A = B = I_5$, and
we vary $K \in \{ \diag(-(1-\rho), -(1-\rho), ..., -(1-0.01)) : \rho \in \{0.1, 0.5, 0.9\} \}$ so that
$L = A+BK = \diag(\rho, \rho, ..., 0.01)$
and $\kappa(P_\infty) = \frac{1}{1-\rho^2} (1-0.01^2)$
for $\rho \in \{0.1, 0.5, 0.9\}$.
Theorem~\ref{thm:lstd_estimate} predicts that as $\rho$ increases towards one,
the number of samples required for $\varepsilon$-relative error increases
as well. Figure~\ref{fig:exp_simple} corroborates this finding.

For our second experiment, we set $B = K = 0_{5 \times 5}$ so that the closed-loop response is simply $A$.
We generate random instances of
$A$ as follows.  For each $\rho \in \{0.1, 0.5, 0.9\}$, we generated $1000$
$A$ instances by setting $A_{ii} = \rho$ for all diagonal entries
and $A_{ij} \sim \textrm{clip($\mathcal{N}(0,
1)$, -1, 1)}$ independently for all upper triangular entries. We order the $A$
instances by $\kappa(P_\infty(A))$, where $A P_\infty(A) A^\T - P_\infty(A) + I =
0$ and take the median. This results in
$\kappa \approx 7$, $\kappa \approx 35$, and $\kappa \approx 7 \times 10^5$
for $\rho = 0.1, 0.5, 0.9$, respectively.
We then run LSTD on the three median instances,
reporting the results in Figure~\ref{fig:exp_trace}.  Once again, as $\kappa(P_\infty(A))$
increases, the required trajectory length increases, as suggested by Theorem~\ref{thm:lstd_estimate}.
We note, however, that Theorem~\ref{thm:lstd_estimate} appears to be conservative in
predicting the actual scaling behavior with $\kappa(P_\infty(A))$.

\renewcommand{\algorithmicrequire}{\textbf{Input:}}
\begin{center}
    \begin{algorithm}[htb]
    \caption{Least-Squares Policy Iteration for LQR}
    \begin{algorithmic}[1]
        \REQUIRE{Starting policy $\pi_0(x) = K_0 x$.}
        \STATE{Collect $M$ rollouts of length $N$: $D = \{ (x^{(\ell)}_k, u^{(\ell)}_k, r^{(\ell)}_k, x^{(\ell)}_{k+1}) \}_{k=1,\ell=1}^{N,M}$, $u^{(\ell)}_k \sim N(0, I)$.}
        \STATE{$K \gets K_0$.}
        \WHILE{not converged}
            \STATE{$\bmattwo{{P_{11}}}{{P_{12}}}{{P_{12}}^\T}{{P_{22}}} \gets \text{LSTD-Q}(D, K)$. \verb|/* Estimate Q-function for K. */|}
            \STATE{$K \gets - {P_{22}}^{-1} {P_{12}}^\T$. \verb|/* Policy improvement step. */|}
        \ENDWHILE
        \STATE{{\bf return} $K$.}
    \end{algorithmic}
    \label{alg:lspi}
    \end{algorithm}
\end{center}

\subsection{Least-Squares Policy Iteration}
\label{sec:experiments:lspi}

We now describe our comparison of the Least-Squares Policy Iteration (LSPI)
algorithm from Lagoudakis and Parr~\cite{lagoudakis03} to the model-based
approaches of Dean et al.~\cite{dean17}.
It is interesting to empirically compare the end-to-end sample complexity of model-free
versus model-based methods for LQR in order to reach a specified
controller cost, since our theoretical results in Section~\ref{sec:lstd:lqr}
suggest that LSPI can require more samples than the model-based approaches.
We look at the same LQR instance from Dean et al., which is
described by
\begin{align}
  A = \begin{bmatrix}
    1.01 & 0.01 & 0 \\
    0.01 & 1.01 & 0.01 \\
    0 & 0.01 & 1.01
  \end{bmatrix} \:, \:\: B = I_{3} \:, \:\: Q = 10^{-3} I_{3} \:, \:\: R = I_{3} \:. \label{eq:dean_system}
\end{align}
We will consider both the discounted LQR problem with $\gamma = 0.98$ and
the average cost LQR problem, given by
\begin{align*}
  \mathop{\mathrm{minimize}}_{U} ~~ \limsup_{T \to \infty} \: \E\left[ \frac{1}{T}\sum_{k=0}^{T-1} (X_k^\T Q X_k + U_k^\T R U_k) \right] ~~\text{s.t.}~~ X_{k+1} = A X_k + B U_k + W_k \:.
\end{align*}
The choice of $\gamma=0.98$ ensures that the closed-loop
system $A+BK$ with $K$ the optimal discounted controller is
stable.
Our metric of interest will be the relative error
$\frac{J(K) - J_\star}{J_\star}$, where $J_\star$ is the optimal infinite-horizon
cost on either the discounted or average cost objective, and
$J(K)$ is the infinite-horizon cost of using the controller $K$ in feedback
with the true system \eqref{eq:dean_system}.

We run our experiments as follows.
We collect $M$ independent trajectories of the system
\eqref{eq:dean_system} excited by independent Gaussian noise
$N(0, I_3)$ of length $N=20$.
This produces a collection of $MN$ tuples
$D = \{(x^{(\ell)}_k, u^{(\ell)}_k, r^{(\ell)}_k, x^{(\ell)}_{k+1})\}_{k=1,\ell=1}^{N,M}$. We repeat this whole process $100$ times.
In our experiments, we will refer to the value $MN$ as the
number of timesteps, and each set $D$ of $MN$ tuples collected
will be referred to as a trial.
As in the previous experiment, we use the prefix of the data
to report different values for the number of timesteps used.
We now describe in more detail
the different algorithms we evaluate.

\begin{center}
    \begin{algorithm}[htb]
    \caption{LSTD-Q}
    \begin{algorithmic}[1]
        \REQUIRE{Samples $D = \{(x_k, u_k, r_k, x_{k+1})\}_{k=1}^{\abs{D}}$, policy $\pi(x) = K x$.}
        \STATE{Define the feature map $\phi(x, u)$ as}
        \begin{align*}
            \phi(x, u) = \svec\left( \vstacktwo{x}{u}\vstacktwo{x}{u}^\T + \eta \vstacktwo{I}{K}\vstacktwo{I}{K}^\T \right) \:.
        \end{align*}
        \STATE{Form the matrices $A$ and $b$ as}
        \begin{align*}
            A = \sum_{k=1}^{\abs{D}} \phi(x_k, u_k)(\phi(x_k, u_k) - \gamma \phi(x_{k+1}, K x_{k+1}))^\T \:, \:\: b = \sum_{k=1}^{\abs{D}} r_k \phi(x_k, u_k) \:.
        \end{align*}
        \STATE{$\widehat{p} \gets A^{\dag} b$.}
        \STATE{{\bf return} $\smat(\widehat{p})$.}
    \end{algorithmic}
    \label{alg:lstdq}
    \end{algorithm}
\end{center}

\paragraph{LSPI.}
For completeness, LSPI is described in
Algorithm~\ref{alg:lspi}.
LSPI relies on a variant of LSTD for $Q$-functions instead of value functions,
which is described in Algorithm~\ref{alg:lstdq}.
To run LSPI, we need a starting controller $K_0$.
The trivial controller $K_0 = 0_{3 \times 3}$ is insufficient, since
the matrix $\sqrt{\gamma} A$ is not stable, and hence
does not induce a finite $Q$-function.
This is a drawback of LSPI; a reasonable initialization must be
chosen for the algorithm to work. For the purposes of comparison,
we set $K_0$ such that
the closed loop matrix $A+BK_0 = \diag(0.6, 0.6, 0.6)$ and is hence a valid
starting point for LSPI.
Furthermore, the relative error
$(J(K_0) - J_\star)/J_\star \approx 6.603$
for the discounted case and
$(J(K_0) - J_\star)/J_\star \approx 4.778$ for the average cost case.
When running LSPI for discounted cost (resp. average cost),
if at any point we estimate a policy $K_t$
such that $\sqrt{\gamma}(A+BK_t)$ (resp. $A+BK_t$) is not stable,
we consider the algorithm as having failed and assign it a score of $+\infty$.

\paragraph{Nominal controller.}

The nominal controller works by first estimating the state-transition matrices
$(\widehat{A}, \widehat{B})$ from the given trajectories via ordinary
least-squares.  With the estimates $(\widehat{A}, \widehat{B})$, we directly
solve via algebraic Ricatti equations for the optimal discounted/average cost controllers under the assumption
that the dynamics are exactly $(\widehat{A}, \widehat{B})$.  We then check to
see if the resulting costs with the nominal controller in feedback with the
true system are finite, and assign a score of $+\infty$ otherwise.

\paragraph{Common Lyapunov controller.}

The common Lyapunov synthesis procedure is developed in Dean et al.\ as a
semidefinite relaxation to the non-convex robust controller synthesis problem
with static state-feedback.
The advantage of the common Lyapunov controller over the nominal is that,
if the program succeeds, it provides a certificate that the actual closed-loop
system is stable (this is not guaranteed by the nominal controller, nor LSPI).
The disadvatage is that this robustness guarantee typically trades off with
performance.
Since the formulation in Dean et al.\ is for the average cost
setting, we only run the procedure in this setting. Because the procedure
is a robust synthesis algorithm, it takes as input an upper bound on the
estimation errors $\norm{\widehat{A} - A} \leq \varepsilon_A$ and
$\norm{\widehat{B} - B} \leq \varepsilon_B$.
We use both the true errors and
$2\times$ the true errors as the input bounds.
The former showcases the best possible performance,
and the latter simulates the
non-parametric bootstrap method used in Dean et al.\ to compute these
confidence bounds; their results suggest that the bootstrap
over-estimates the true errors by roughly a factor of two.
We solve the resulting semidefinite programs using
cvxpy~\cite{diamond16} with MOSEK~\cite{mosek} as the backend solver.

\begin{figure}[t!]
  \centering
  \begin{minipage}[t]{0.46\textwidth}
  \begin{center}
  \includegraphics[width=\columnwidth]{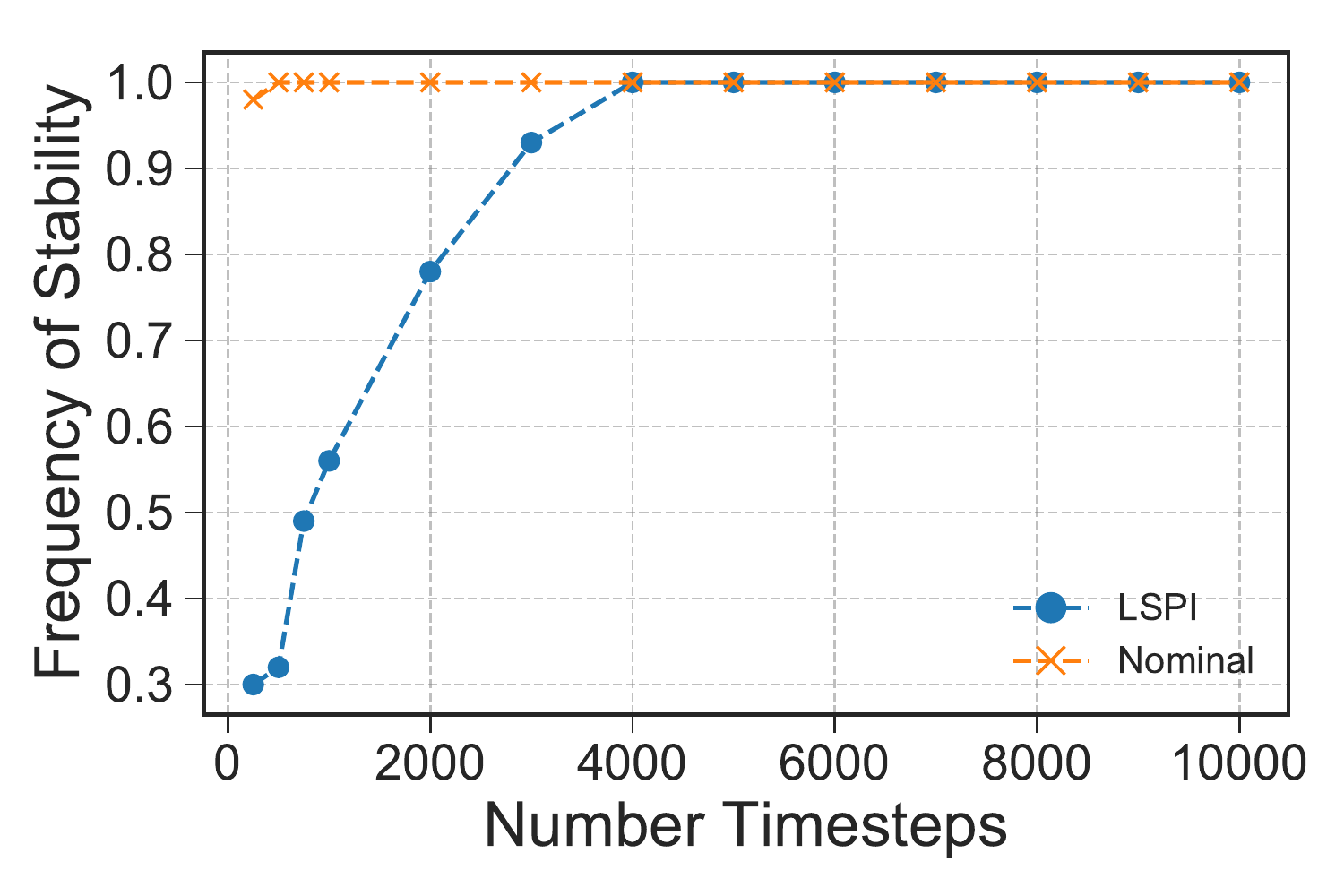}
  \end{center}
  \caption{A comparison of how frequently (out of $100$ trials)
	both LSPI and the nominal synthesis procedure were able to produce
	a controller $\widehat{K}$ such that the matrix
	$\sqrt{\gamma}(A+B\widehat{K})$ was stable. This condition is necessary
	and sufficient for the discounted infinite-horizon cost to be finite.}
  \label{fig:dis_cost_freq_stability}
  \end{minipage}
  \hspace{.02\textwidth}
  \begin{minipage}[t]{0.46\textwidth}
  \begin{center}
	\includegraphics[width=\columnwidth]{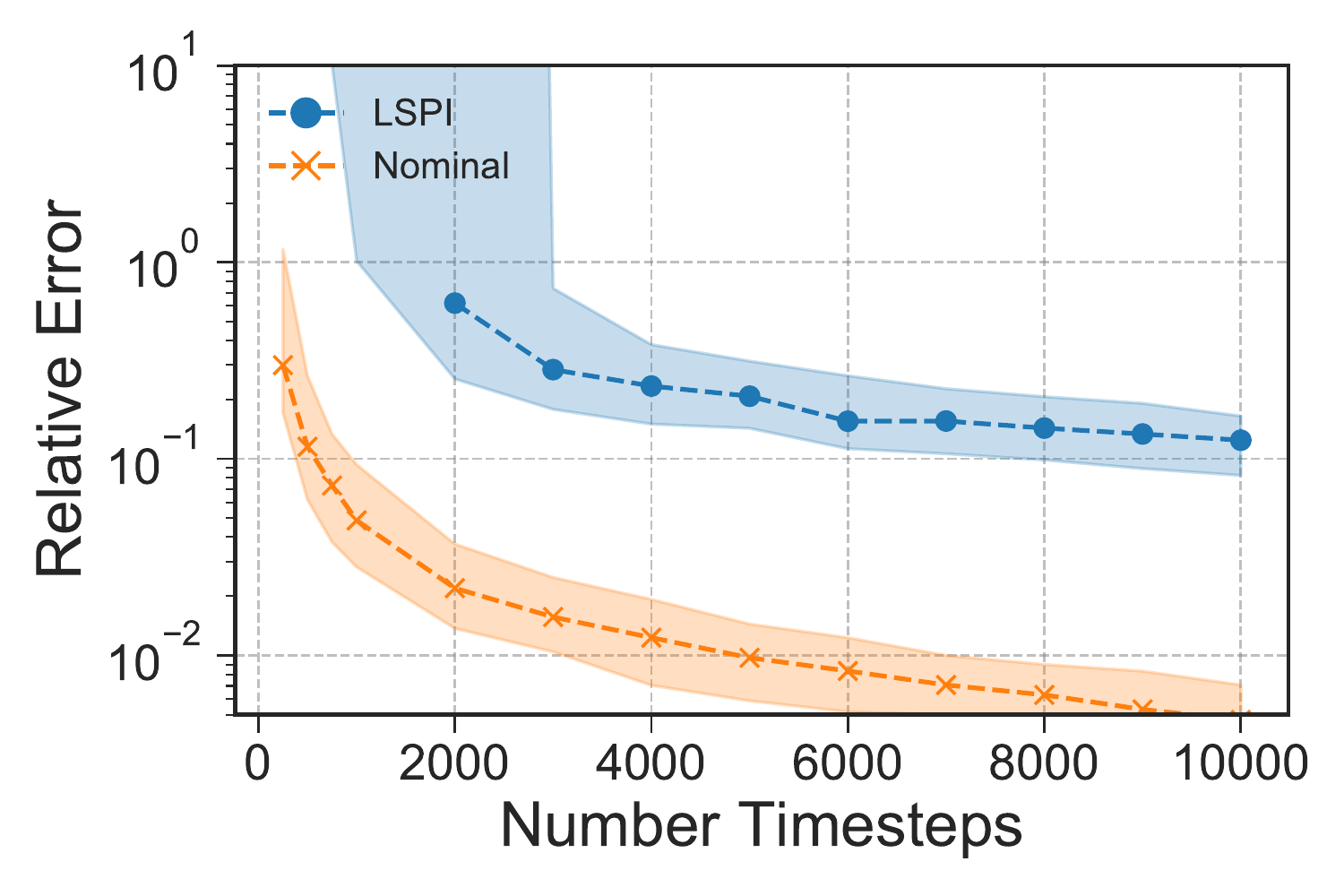}
  \end{center}
  \caption{A comparison of the relative error of the controllers
	produced by both LSPI and the nominal synthesis procedure for the
	discounted LQR problem.
	The points along the dashed line denote the median cost, and the
	shaded region covers the $25$-th to $75$-th percentile out of $100$ trials.}
  \label{fig:dis_cost_relative_error}
  \end{minipage}
\end{figure}

The results for the discounted LQR problem are shown in Figure~\ref{fig:dis_cost_freq_stability}
and Figure~\ref{fig:dis_cost_relative_error}, and
the results for the average cost LQR problem are shown in
Figure~\ref{fig:avg_cost_freq_stability}
and Figure~\ref{fig:avg_cost_relative_error}.
We observe on the discounted problem that LSPI
less robust and more sample inefficient than the nominal controller.
In Figure~\ref{fig:dis_cost_freq_stability}, we observe that even with $3000$
timesteps the frequency of stability for LSPI is worse than that
of the nominal controller at $250$ timesteps.
Similarly, in Figure~\ref{fig:dis_cost_relative_error}, we see that the
relative error achieved by LSPI at $3000$ timesteps is comparable to that
achieved by the nominal controller at $250$ timesteps.
The qualitative differences between LSPI and the nominal controller
remain the same when we move to the average cost controller.
In Figure~\ref{fig:avg_cost_relative_error},
we see that the nominal controller and the common Lyapunov
controller given the actual error bounds perform the best,
the common Lyapunov controller given $2\times$ the actual error bound
performs slightly worse, and
the performance of LSPI is substantially behind the rest, taking for instance over
$10\times$ more samples compared to the nominal controller
to achieve a relative error of $10^{-1}$ .

\begin{figure}[t!]
  \centering
  \begin{minipage}[t]{0.46\textwidth}
  \begin{center}
  \includegraphics[width=\columnwidth]{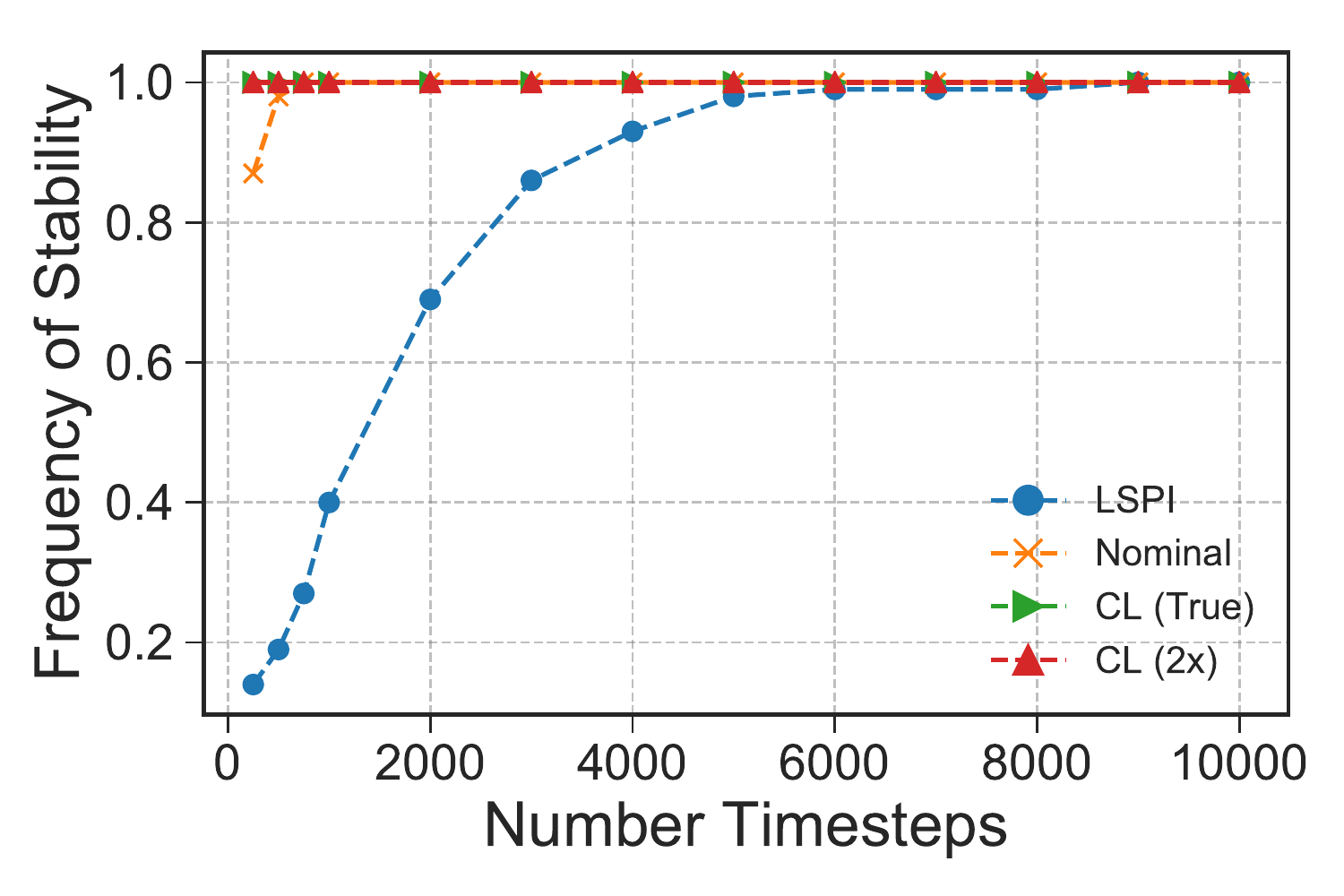}
  \end{center}
  \caption{A comparison of how frequently (out of $100$ trials)
	LSPI, the nominal synthesis procedure, and the common Lyapunov (CL)
	synthesis procedures were able to produce
	a controller $\widehat{K}$ such that the matrix
	$A+B\widehat{K}$ was stable.
	This condition is necessary
	and sufficient for the average infinite-horizon cost to be finite.}
  \label{fig:avg_cost_freq_stability}
  \end{minipage}
  \hspace{.02\textwidth}
  \begin{minipage}[t]{0.46\textwidth}
  \begin{center}
	\includegraphics[width=\columnwidth]{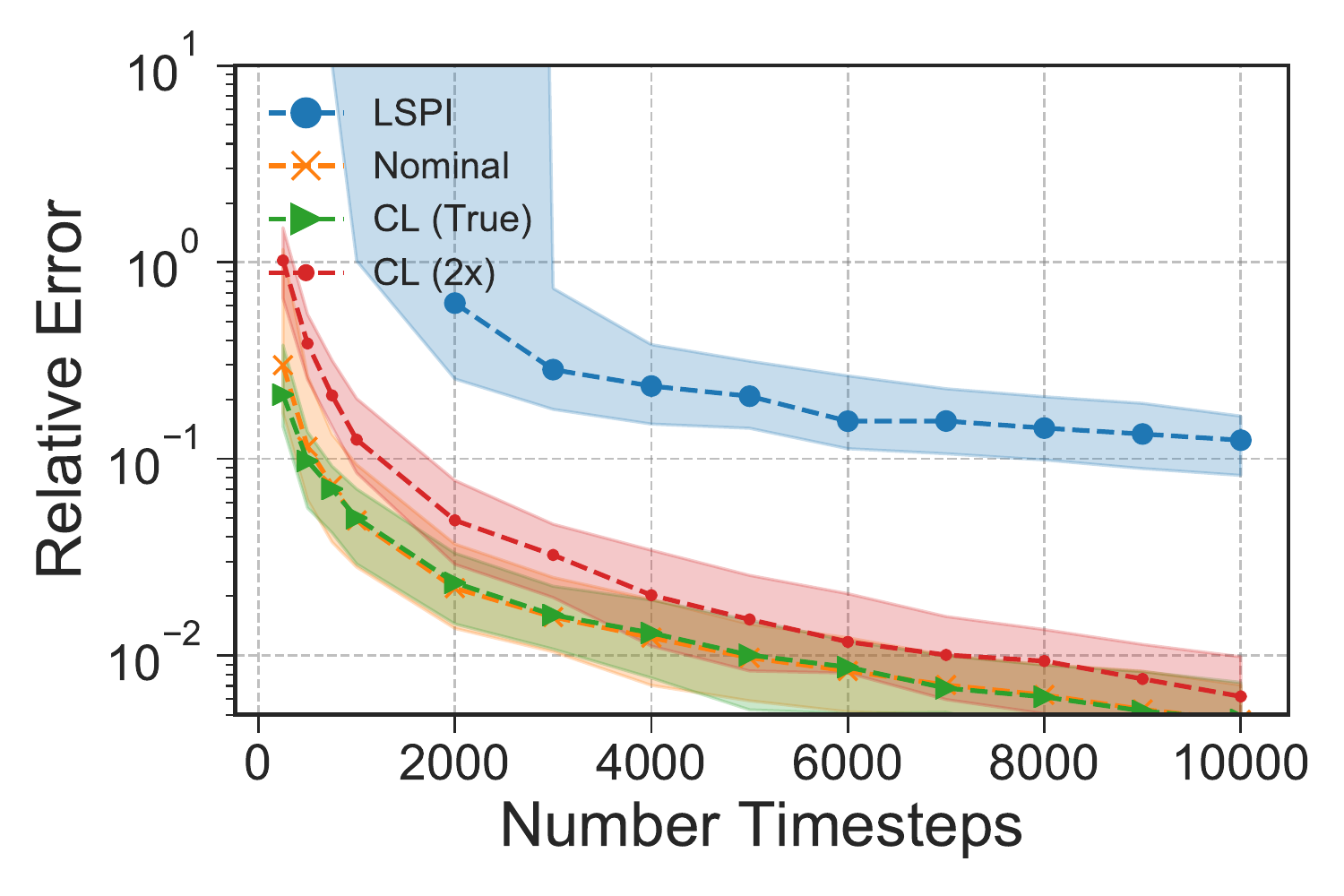}
  \end{center}
  \caption{A comparison of the relative error of the controllers
	produced by LSPI, the nominal synthesis procedure, and the common Lyapunov (CL) procedures for the
	average cost LQR problem.
	The points along the dashed line denote the median cost, and the
	shaded region covers the $25$-th to $75$-th percentile out of $100$ trials.}
  \label{fig:avg_cost_relative_error}
  \end{minipage}
\end{figure}

\section{Conclusion}

We studied the number of samples needed for the LSTD estimator to return a
$\varepsilon$-accurate solution in relative error for the value function
associated to a fixed policy $\pi$ for LQR.  In the process of deriving
our result, we provided a concentration result for the minimum eigenvalue of a
sample covariance matrix formed along the trajectory of a $\beta$-mixing
stochastic process. 
Empirically, we demonstrated that model-free policy iteration (LSPI) requires
substantially more samples on certain LQR instances than the model-based
methods from Dean et al.
We hope our results encourage further investigation
into the foundations of RL for continuous control problems.
We now highlight some possible extensions of our work.

\paragraph{End-to-end guarantees.}
Theorem~\ref{thm:lstd_estimate} provides an upper bound on the estimation error
of the value function for a fixed policy. 
While extending our analysis to estimating a fixed $Q$-function is
straightforward, it is not as clear how to iterate this process. In particular,
how many samples are needed until policy-iteration reaches a policy which
receives an expected reward that is within $\varepsilon$ additive or relative
error to the optimal value? Our experiments in Section~\ref{sec:experiments:lspi}
suggest that the answer may be more than model-free methods, but 
it is not clear if this phenomenon is general or if there are instances
where LSPI outperforms model-based methods.
Can we establish conditions under which model-based methods will 
always outperform model-free methods?

\paragraph{Lower bounds.}
Another interesting question is how sharp the bound in 
Theorem~\ref{thm:lstd_estimate} is, especially in terms of its
dependence on the spectral properties of $P_\infty$.
While our numerical experiments in Section~\ref{sec:experiments:synthetic}
suggest that the qualitative behavior is correct, we do not 
have any algorithmic lower bounds for LSTD which confirm 
this rigorously.
Furthermore, what are the information-theoretic lower bounds
incurred by any RL algorithm for the LQR problem in terms of number of samples?

\paragraph{Other model-free RL estimators.}
Policy gradient algorithms such as 
Trust Region Policy Optimization~\cite{schulman15} 
have become increasingly popular for solving MDPs in robotics.
How does the sample complexity of policy gradient or
TRPO compare to LSPI and the model-based methods of Dean et al.?

\section*{Acknowledgements}
We thank Orianna DeMasi, Vitaly Kuznetsov, Horia Mania, Max Simchowitz, Vikas Sindhwani, and Xinyan Yan
for many helpful comments and suggestions.
Part of this work was completed when ST was interning at Google Brain, New
York, NY.  BR is generously supported by NSF award CCF-1359814, ONR awards
N00014-14-1-0024 and N00014-17-1-2191, the DARPA Fundamental Limits of Learning
(Fun LoL) Program, a Sloan Research Fellowship, and a Google Faculty Award.

{
\small
\bibliographystyle{abbrv}
\bibliography{paper}
}

\appendix
\section{Gaussian Moment Lemmas}

First, we present an elementary claim regarding the fourth moment of a non-isotropic multivariate Gaussian.
For completeness, we provide a proof.
\begin{prop}
\label{prop:gaussian_fourth_moment}
Let $x \sim \mathcal{N}(0, I)$, and $A, B$ two fixed symmetric matrices.
We have that
\begin{align*}
	\E[x^\T A x x^\T B x] = 2 \ip{A}{B} + \Tr(A)\Tr(B) \:.
\end{align*}
\end{prop}
\begin{proof}
First, let $u, v$ be two fixed vectors.
Let us compute
\begin{align*}
	\E [ (x^\T u)^2 (x^\T v)^2 ] = v^\T \E[ (x^\T u)^2 xx^\T ]  v \:.
\end{align*}
Fix an $i \neq j$.
We have that
\begin{align*}
	\E\left[\left(\sum_{k} x_k u_k\right)^2 x_i x_j\right] = \sum_{k,l} \E[ x_i x_j x_k x_l u_k u_l ] = 2 u_i u_j \E[ x_i^2 x_j^2 ] = 2 u_i u_j \:.
\end{align*}
On the other hand, for a fixed $i$,
\begin{align*}
	\E\left[\left(\sum_{j} x_j u_k\right)^2 x_i^2 \right] = \sum_{j, k} \E[x_i^2 x_j x_k u_j u_k] = u_i^2 \E[x_i^4] + \sum_{j \neq i} u_j^2 \E[x_j^2] = 3 u_i^2 + \sum_{j \neq i} u_j^2 = 2 u_i^2 + \norm{u}^2 \:.
\end{align*}
Hence,
\begin{align*}
	\E[(x^\T u)^2 xx^\T] = \norm{u}^2 I + 2 uu^\T \:.
\end{align*}
From this we conclude that
\begin{align*}
	\E [ (x^\T u)^2 (x^\T v)^2 ] = \norm{u}^2 \norm{v}^2 + 2 (u^\T v)^2 \:.
\end{align*}
Now write the eigen-decompositions of $A$ and $B$ as
$A = \sum_{i} \lambda_i u_iu_i^\T$ and $B = \sum_{i} \gamma_i v_iv_i^\T$.
We have that
\begin{align*}
	x^\T A x x^\T B x = \left(\sum_{i} \lambda_i (x^\T u_i)^2 \right)\left(\sum_{i} \gamma_i (x^\T v_i)^2 \right) = \sum_{i,j} \lambda_i \gamma_j (x^\T u_i)^2 (x^\T v_j)^2 \:.
\end{align*}
Taking expectations,
\begin{align*}
	\E[ x^\T A x x^\T B x ] &= \sum_{i,j} \lambda_i \gamma_j \E[ (x^\T u_i)^2 (x^\T v_j)^2 ] \\
	&= \sum_{i,j} \lambda_i \gamma_j ( 1 + 2(u_i^\T v_j)^2 ) \\
	&= \left(\sum_{i} \lambda_i\right)\left(\sum_{j} \lambda_j\right) + 2 \Tr\left( \sum_{i,j} \lambda_i \gamma_j u_iu_i^\T v_jv_j^\T \right) \\
	&= \Tr(A) \Tr(B) + 2 \Tr(A B) \:.
\end{align*}
\end{proof}
Next, we state a well-known result
regarding Gaussian hypercontractivity.
\begin{lem}[See e.g. Bogachev~\cite{bogachev15}]
\label{lem:gaussian_hypercontractivity}
Let $f$ be a degree $d$ polynomial and $x \sim \mathcal{N}(0, I)$. For any $q > 2$, we have
\begin{align*}
  \norm{f}_{L^q} \leq (q-1)^{d/2} \norm{f}_{L^2} \:.
\end{align*}
\end{lem}

\section{Proof of Lemma~\ref{lemma:structural_bound_lstd_lqr}}
\label{sec:appendix:lstd}

This follows the development of Lazaric et al.~\cite{lazaric12}.
From a given trajectory $\{(X_k, R_k, X_{k+1})\}_{k=1}^{N+1}$,
let us define three $N \times d$ matrices $\Phi$, $\Phi_+$, and $\Psi$ as follows:
\begin{align*}
  \Phi = \begin{bmatrix} -\phi(X_1)^\T - \\ \vdots \\ -\phi(X_N)^\T- \end{bmatrix} \:, \:\:
  \Phi_+ = \begin{bmatrix} -\phi(X_2)^\T - \\ \vdots \\ -\phi(X_{N+1})^\T- \end{bmatrix} \:, \:\:
  \Psi = \begin{bmatrix} -\psi(X_1)^\T- \\ \vdots \\ -\psi(X_N)^\T- \end{bmatrix} \:.
\end{align*}
Above, $\psi(x) := \E_{x' \sim p(\cdot |x, \pi(x))}[ \phi(x') ]$, where $p(\cdot | x, a)$
is the transition dynamics of the MDP at state $x$ with action $a$.
The LSTD estimator is to find a $\widehat{w} \in \R^d$ such that
\begin{align*}
  \Phi^\T( \Phi - \gamma \Phi_+) \widehat{w} = \Phi^\T R \:,
\end{align*}
where $R = (R_1, ..., R_N) \in \R^N$ is the vector of rewards received.
Under the linear-architecture assumption, Bellman's equation~\eqref{eq:bellman_general} implies that
\begin{align}
  (\Phi - \gamma \Psi) w_* = R \:. \label{eq:bellman_equation}
\end{align}
Define the shift operator $\widehat{P} : \R^{N + 1} \longrightarrow \R^{N + 1}$
as
\begin{align*}
  (\widehat{P} y)_t = \begin{cases} y_{t+1} &\text{if } 1 \leq t \leq N \\
  0 &\text{if } t = N + 1 \end{cases} \:,
\end{align*}
and define the empirical Bellman operator
$\widehat{T} : \R^{N + 1} \longrightarrow \R^{N+1}$ as
\begin{align*}
  \widehat{T}(y) = R_e + \gamma \widehat{P} y \:, \:\: R_e = \begin{bmatrix} R \\ 0 \end{bmatrix} \in \R^{N+1} \:.
\end{align*}
Let $\Phi_e = \begin{bmatrix} \Phi \\ 0 \end{bmatrix} \in \R^{(N + 1) \times d}$.
We now see that the operator $P_{\Phi_e} \widehat{T}$ is contractive,
where $P_{\Phi_e}$ denotes the orthogonal projector onto the range of $\Phi_e$.
\begin{prop}
\label{prop:proj_bellman_contractive}
For all $y, z \in \R^{N+1}$, we have
\begin{align*}
  \norm{ P_{\Phi_e} \widehat{T}(y) - P_{\Phi_e}\widehat{T}(z) } \leq \gamma \norm{y - z} \:.
\end{align*}
\end{prop}
\begin{proof}
By definition, we have
$\widehat{T}(y) - \widehat{T}(z) = \gamma \widehat{P} (y - z)$.
By construction, we have that $\norm{\widehat{P} (y-z)} \leq \norm{y-z}$.
The claim now follows since the projection operator $P_{\Phi_e}$ is non-expansive in
the $\ell_2$-norm.
\end{proof}
Hence, by Banach's fixed-point theorem, the operator $P_{\Phi_e} \widehat{T}$ has a unique
fixed-point.
It turns out the LSTD estimator is solving for this fixed point, as the
following proposition demonstrates.
\begin{prop}[Section 5.2, Lagoudakis and Parr~\cite{lagoudakis03}]
\label{prop:lstd_fixed_point}
Suppose that $\Phi$ has full column rank, and that
$w \in \R^d$ satisfies
\begin{align*}
  \Phi^\T(\Phi - \gamma \Phi_+) w = \Phi^\T R \:.
\end{align*}
Then, we have that the fixed-point equation holds
\begin{align*}
  \Phi_e w = P_{\Phi_e} \widehat{T}(\Phi_e w) \:.
\end{align*}
\end{prop}
\begin{proof}
First, we observe the following equivalences
\begin{align*}
  \Phi^\T( \Phi - \gamma \Phi_+) w = \Phi^\T R &\Longleftrightarrow \Phi^\T \Phi w = \Phi^\T(R + \gamma \Phi_+ w) \\
  &\Longleftrightarrow w = (\Phi^\T \Phi)^{-1} \Phi^\T(R + \gamma \Phi_+ w) \\
  &\Longleftrightarrow \Phi w = \Phi(\Phi^\T \Phi)^{-1} \Phi^\T(R + \gamma \Phi_+ w) \\
  &\Longleftrightarrow \Phi w = P_{\Phi} (R + \gamma \Phi_+ w) \:.
\end{align*}
Next, it is easy to see that
\begin{align*}
  P_{\Phi_e} = \begin{bmatrix} P_{\Phi} & 0 \\ 0 & 0 \end{bmatrix} \:.
\end{align*}
Hence, the following relation holds
\begin{align*}
  \Phi_e w = \begin{bmatrix} \Phi w \\ 0 \end{bmatrix} = \begin{bmatrix} P_{\Phi} & 0 \\ 0 & 0 \end{bmatrix} \left( \begin{bmatrix} R \\ 0 \end{bmatrix} + \gamma \widehat{P} w \right) = P_{\Phi_e} \widehat{T}(\Phi_e w) \:.
\end{align*}
\end{proof}
Next is a structural result for the LSTD estimator.
\begin{prop}[Theorem 1, Lazaric et al.~\cite{lazaric12}]
\label{prop:structural}
Let $\widehat{w}$ denote the LSTD estimator and suppose that $\Phi$ has full column rank.
We have that
\begin{align*}
  \norm{\Phi w_* - \Phi \widehat{w}} \leq \frac{\gamma}{1-\gamma} \norm{ P_{\Phi}( \Phi_+ - \Psi) w_* } \:.
\end{align*}
\end{prop}
\begin{proof}
First, observe that
\begin{align*}
  \norm{\Phi w_* - \Phi \widehat{w}} &= \norm{\Phi_e w_* - \Phi_e \widehat{w}} \\
  &= \norm{P_{\Phi_e} \Phi_e w_* - \Phi_e \widehat{w}} \\
  &\leq \norm{P_{\Phi_e} \Phi_e w_*  - P_{\Phi_e} \widehat{T}(\Phi_e w_*) } + \norm{ P_{\Phi_e} \widehat{T}( \Phi_e w_* ) -\Phi_e \widehat{w} } \\
  &\stackrel{(a)}{=} \norm{P_{\Phi_e} \Phi_e w_* - P_{\Phi_e} \widehat{T} (\Phi_e w_*) } + \norm{ P_{\Phi_e} \widehat{T}( \Phi_e w_*) - P_{\Phi_e} \widehat{T}(\Phi_e \widehat{w}) } \\
  &\stackrel{(b)}{\leq} \norm{P_{\Phi_e} \Phi_e w_* - P_{\Phi_e} \widehat{T} (\Phi_e w_*) } + \gamma \norm{ \Phi_e w_* - \Phi_e \widehat{w} } \\
  &= \norm{P_{\Phi_e} \Phi_e w_* - P_{\Phi_e} \widehat{T} (\Phi_e w_*) } + \gamma \norm{ \Phi w_* - \Phi \widehat{w}} \:.
\end{align*}
Above,
(a) uses the fact that the LSTD estimator satisfies
the fixed-point equation from Proposition~\ref{prop:lstd_fixed_point}, and
(b) uses the $\gamma$-contractive property of the
$P_{\Phi_e} \widehat{T}$ operator from Proposition~\ref{prop:proj_bellman_contractive}.
At this point, we have shown that
\begin{align*}
  \norm{\Phi w_* - \Phi \widehat{w}} \leq \frac{1}{1-\gamma} \norm{ P_{\Phi_e} (\widehat{T} (\Phi_e w_*) - \Phi_e w_*) } \:.
\end{align*}
To finish the proof, we note that
\begin{align*}
  \widehat{T} (\Phi_e w_*) - \Phi_e w_* &= R_e + \gamma \widehat{P} \Phi_e w_* - \Phi_e w_* \\
  &\stackrel{(a)}{=} \begin{bmatrix} \Phi - \gamma \Psi \\ 0 \end{bmatrix} w_* + \begin{bmatrix} \gamma \Phi_+ \\ 0 \end{bmatrix} w_* - \begin{bmatrix} \Phi \\ 0 \end{bmatrix} w_* \\
    &= \begin{bmatrix} \gamma( \Phi_+ - \Psi ) w_* \\ 0 \end{bmatrix} \:,
\end{align*}
where (a) comes from Bellman's equation \eqref{eq:bellman_equation}.
The claim now follows.
\end{proof}
Lemma~\ref{lemma:structural_bound_lstd_lqr}
follows from Proposition~\ref{prop:structural} by making two observations.
Recall that we assume $\Phi$ has full column rank.
Therefore, for any vector $v \in \R^d$,
\begin{align*}
  \norm{\Phi v} = \sqrt{v^\T \Phi^\T \Phi v} \geq \sqrt{\lambda_{\min}(\Phi^\T \Phi)} \norm{v} \:.
\end{align*}
Furthermore, for any vector $\xi \in \R^N$,
\begin{align*}
  \norm{P_{\Phi} \xi} = \sqrt{ \xi^\T P_{\Phi} \xi } = \sqrt{ \xi^\T \Phi (\Phi^\T \Phi)^{-1} \Phi^\T \xi } \leq \frac{\norm{ \Phi^\T \xi}}{\sqrt{\lambda_{\min}(\Phi^\T \Phi)}} \:.
\end{align*}
Combining these two inequalities with Proposition~\ref{prop:structural},
we have
\begin{align*}
  \norm{w_* - \widehat{w}} \leq \frac{\gamma}{1-\gamma} \frac{\norm{\Phi^\T (\Phi_+ - \Psi) w_*}}{\lambda_{\min}(\Phi^\T \Phi)} \:.
\end{align*}

\end{document}